\documentclass{article}

\usepackage{microtype}
\usepackage{graphicx}
\usepackage{subcaption}
\usepackage{booktabs} 
\usepackage{multirow}
\usepackage{makecell}

\usepackage{hyperref}

\usepackage[preprint]{icml2026}

\usepackage{amsmath}
\usepackage{amssymb}
\usepackage{mathtools}
\usepackage{amsthm}

\usepackage[capitalize,noabbrev]{cleveref}

\theoremstyle{plain}
\newtheorem{theorem}{Theorem}[section]
\newtheorem{proposition}[theorem]{Proposition}

\theoremstyle{definition}
\newtheorem{definition}[theorem]{Definition}

\theoremstyle{remark}

\usepackage[textsize=tiny]{todonotes}

\icmltitlerunning{
Learning Additively Compositional Latent Actions for Embodied AI}

\begin{document}

\twocolumn[
  \icmltitle{
  Learning Additively Compositional Latent Actions for Embodied AI}



  \icmlsetsymbol{equal}{*}

  \begin{icmlauthorlist}
    \icmlauthor{Hangxing Wei}{whu}
    \icmlauthor{Xiaoyu Chen}{thu}
    \icmlauthor{Chuheng Zhang}{comp}
    \icmlauthor{Tim Pearce}{comp}
    \icmlauthor{Jianyu Chen}{thu}
    \icmlauthor{Alex Lamb}{thu}
    \icmlauthor{Li Zhao}{comp}
    \icmlauthor{Jiang Bian}{comp}
  \end{icmlauthorlist}

  \icmlaffiliation{whu}{Wuhan University}
  \icmlaffiliation{thu}{Tsinghua University}
  \icmlaffiliation{comp}{Microsoft Research}

  \icmlcorrespondingauthor{Li Zhao}{lizo@microsoft.com}

  \icmlkeywords{Machine Learning, ICML}

  \vskip 0.3in
]



\printAffiliationsAndNotice{}  

\begin{abstract}
%
Latent action learning infers pseudo-action labels from visual transitions, providing an approach to leverage internet-scale video for embodied AI.
%
However, most methods learn latent actions without structural priors that encode the additive, compositional structure of physical motion.
%
As a result, latents often entangle irrelevant scene details or information about future observations with true state changes and miscalibrate motion magnitude.
%
We introduce \textit{Additively Compositional Latent Action Model} (AC-LAM), which enforces scene-wise additive composition structure over short horizons on the latent action space. 
These AC constraints encourage simple algebraic structure in the latent action space~(identity, inverse, cycle consistency) and 
suppress information that
does not compose additively.
Empirically, AC-LAM learns more structured, motion-specific, and displacement-calibrated latent actions and provides stronger supervision for downstream policy learning, outperforming state-of-the-art LAMs across simulated and real-world tabletop tasks. 
\end{abstract}

\section{Introduction}

Latent action learning has emerged as a scalable paradigm in embodied AI, enabling pretraining from internet-scale video by deriving pseudo-action labels from visual transitions~\citep{bruce2024genie,ye2024latent}.
However, existing methods lack structural priors that reflect the compositional nature of physical motion.
Consequently, learned latents often entangle irrelevant information (e.g., scene details, future observations) with pure state changes and miscalibrate motion magnitude,
weakening transfer, planning, and generalization. 
As latent actions encode motion semantics, it is natural to expect their norm to reflect motion magnitude.
However, 
lacking explicit structural constraints, latent action learning methods usually fail to capture this relationship. As shown in Figure~\ref{fig:latent-norm-trajectory} for a real‑robot trajectory, the latent action norm $\|LAM(o_0, o_t)\|$ from LAPA~\citep{ye2024latent}
is systematically under‑calibrated for motion magnitude~(displacement).
Accurately calibrating displacement and enforcing compositional structure for latent actions is thus critical for robust control and generalization.
\begin{figure}[t]
    \centering
    \includegraphics[width=0.48\textwidth]{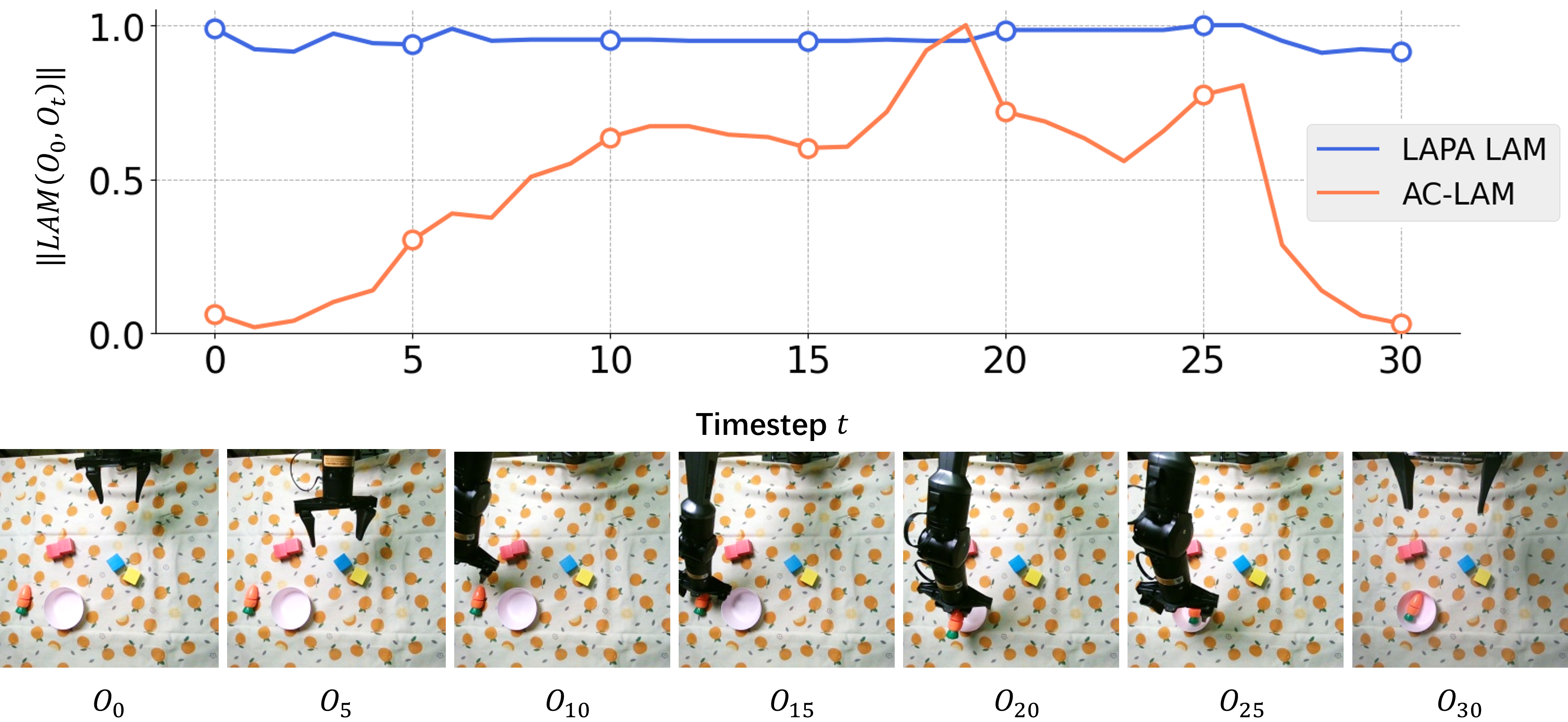}
    \caption{
    Evolution of the normalized latent action norm $\|LAM(o_0, o_t)\|$ over time intervals $t$. The figure shows that, compared with baselines, our AC constraints (AC-LAM) induce displacement-calibrated latent actions, effectively capturing the magnitude of the transition from $o_0$ to $o_t$.
    }
    \label{fig:latent-norm-trajectory}
\end{figure}

We address these limitations by introducing the \emph{Additively Compositional Latent Action Model} (AC-LAM), which enforces scene-wise additive composition structure. 
Concretely, latent actions inferred between observations $o_i,o_j,o_k$ from the same scene satisfy
\[
z_{ik} = z_{ij} + z_{jk},
\]
aligning the representation with the linear and compositional structure of short-horizon motion semantics. This structural prior regularizes the latent space toward more structured latent actions that compose additively within the same scene. 
While this prior may not hold for all possible actions, basic motion primitives often compose additively (e.g. `move right' then `move up', can be composed as `move diagonally right and up').


Concretely, we show that AC constraints 
encourage desirable algebraic structure in the latent action space: 
identity ($z_{ii}=0$) and inverse ($z_{ji}=-z_{ij}$) elements, and cycle consistency~($z_{ij} + z_{jk} + z_{ki} = 0$). 
Moreover, under simplified assumptions, the constraints 
suppress information that does not compose additively:
static environment terms and future leakage are suppressed because they violate additive consistency across triples. 
While linearity is an approximation in practice, AC acts as a regularizer that promotes motion-centric latents and reduces entanglement with non-compositional factors.
Based on the AC prior, we propose AC-LAM, a novel LAM that enforces AC constraints through a compositional loss in the forward-dynamics (FDM) model, denoted $\mathcal{L}_{\mathrm{AC\text{-}FDM}}$. For triples $(o_i,o_j,o_k)$ from the same trajectory, we reconstruct $o_k$ by decoding from the summed latents $z_{ij}+z_{jk}$ (Eq.~\ref{eq:ac-fdm-impl}). The $\mathcal{L}_{\mathrm{AC\text{-}FDM}}$ loss operates on post-VQ continuous embeddings, preserving displacement calibration while retaining the VQ bottleneck.
We find $\mathcal{L}_{\mathrm{AC\text{-}FDM}}$ is empirically more stable than a version of the loss implemented in the IDM (Eq.~\ref{eq:ac-idm-impl}), and can be optimized jointly with standard reconstruction and bottleneck terms over latent actions sampled within short horizons.

Our empirical results show that AC-LAM learns more interpretable, motion-specific, and displacement-calibrated latent actions: they correlate more strongly with true displacement~(as shown in Figure \ref{fig:latent-norm-trajectory}), adhere more closely to identity/inverse behavior, and reduce information captured about environment aesthetics and future observations. Using latents generated by AC-LAM as 
training supervision targets improves downstream policy learning compared to state-of-the-art latent action models across simulated and real-world tabletop tasks. In short, by incorporating a prior reflecting the nature of physical motion, AC-LAM provides advantages over unstructured latents.


Our main contributions are:
\begin{itemize}
\item \textbf{Structural prior}: We introduce a scene-wise additive composition prior over short horizons, encouraging \(z_{ik} \approx z_{ij} + z_{jk}\) in the latent action space.
\item \textbf{Analysis}: We show that AC gives rise to identity and inverse elements, yields cycle consistency, and suppresses information that
does not compose additively.
\item \textbf{Practice}: We enforce AC through a compositional loss
in the forward-dynamics model, jointly optimized with reconstruction/bottleneck losses and short-horizon sampling.
\item \textbf{Empirics}: 
AC-LAM yields more structured latents which provide stronger policy supervision across simulation and real robots, surpassing  state-of-the-art LAMs.
\end{itemize}
\section{Related Work}

\paragraph{Latent Action Learning}

Latent action learning abstracts temporal dynamics in video by modeling inter-frame visual change. Early works in discrete settings (e.g., LAPO~\citep{schmidt2023learning}, Genie~\citep{bruce2024genie}) extract latent actions in 2D platformer games. Recent approaches extend to human/robot videos for continuous control (e.g., LAPA~\citep{ye2024latent}, IGOR~\citep{chen2024igor}, MotoGPT~\citep{chen2024moto}), often using VQ-VAE to discretize latents. UniVLA~\citep{bu2025univlalearningacttaskcentric} learns task-centric action representations via language conditioning. Continuous latent actions with VAE regularization instead of VQ (e.g., CLAM~\citep{liang2025clam}, COMO~\citep{yang2025como}, LAWM~\citep{garrido2026}) offer alternative bottlenecks. Optical-flow-based methods (e.g., Motus~\citep{bi2025motus}, ViPRA~\citep{routray2025vipra}, LAOF~\citep{bu2025laof}) emphasize motion-centric latents. Label-supervised designs (e.g., LAOM~\citep{nikulin2025latent}, Linear LAM~\citep{zhang2025what}, villa-X~\citep{chen2025villax}, CLAP~\citep{zhang2026clap}) primarily leverage action labels and/or proprioceptive states to suppress distractor-induced variations. Temporally extended latents are explored in VideoWorld~\citep{ren2025videoworld} and SSM-VLA~\citep{cai2025SSMVLA}, and viewpoint-invariant latents are explored in MVP-LAM~\citep{MVP-LAM}.
While prior LAMs span discretization, conditioning, supervision, motion cues, and temporal extension, AC-LAM is the first to explicitly endow the latent action space with an additive compositional prior, yielding structured latents that strengthen downstream policy learning.
\vspace{-0.3cm}
\paragraph{Compositionality in Deep Learning}
Compositionality has been extensively studied in embeddings, starting with word vectors where linear semantic arithmetic holds (e.g., vec(“Russia”) + vec(“river”) $\approx$ vec(“Volga River”)) and extending to paraphrase and sentence embeddings \citep{mikolov2013distributed,mikolov2013efficient,wieting2015paraphrase,AroraLM17}, as well as vision–language representations that exhibit approximately linear subspaces and compositional generalization \citep{trager2024linear,berasi2025text}. The most closely related to our work is Adaworld \citep{gao2025adaworld}, which demonstrates latent action composition in games with a discrete action space learned without explicit structural constraints. CoLA-World~\cite{wang2025cola_world} also learns a latent-action world model but does not address latent action composition. 
In contrast, we focus on latent action learning for embodied AI with continuous control, enforcing scene-wise additive composition and introducing a structural prior over the latent action space.

\section{Method}

We present the \emph{Additively Compositional Latent Action Model} (AC-LAM), 
which enforces scene-wise additivity on the latent action space.
This structural prior aligns the latent action with the physical nature of robotic motions,
yielding more interpretable, motion-specific, and displacement-calibrated
latent actions.
Throughout, we use \emph{additive} and \emph{additively compositional} interchangeably.

\subsection{Problem Formulation}

Let \(\mathcal{O}\) denote the observation space (e.g., images), 
and let \(\mathcal{Z} \subseteq \mathbb{R}^d\) denote the latent action space. 
Each observation $o \in \mathcal{O}$ belongs to a scene denoted by $e(o)$.
\begin{definition}[Scene]
Two observations $o_i, o_j$ are considered to be within the same scene, $e(o_i)=e(o_j)$, if some sequence of actions exists that the agent can take to get between $o_i$ and $o_j$. Intuitively, a scene contains information about embodiment, background, objects, and so on.
\label{def_scene}
\end{definition}

Our goal is to learn a latent action model (LAM) which consists of: 
\begin{itemize}
    \item An inverse dynamics model~(IDM) 
    \(f: \mathcal{O} \times \mathcal{O} \to \mathcal{Z}\) 
    such that latent action \(z_{ij} := f(o_i, o_j)\) encodes the transformation from 
    \(o_i\) to \(o_j\), where $e(o_i)=e(o_j)$.
    \item A forward dynamics model~(FDM) \(F: \mathcal{Z} \times \mathcal{O} \to \mathcal{O}\) 
    such that $o_j = F(o_i, z_{ij})$.
    We also use the notation $F_{z}(o) := F(o,z)$. 
\end{itemize}

LAMs are generally trained using a reconstruction loss.
\begin{equation}
\label{eq:rec_loss}
\mathcal{L}_{\mathrm{rec}} \;=\; \mathbb{E}_{i,j}\,\big\|\,o_j - F_{z_{ij}}(o_i)\big\|_2^2
\end{equation}

\subsection{Additively Compositional Latent Actions}
\label{subsec:AC-LA}

We enforce the \emph{(scene-wise) additive structure} on the latent actions induced within the same scene.
\begin{equation}
\begin{aligned}
    z_{ik} & = z_{ij} + z_{jk}, \\
    & \forall o_i, o_j, 
    o_k, \ s.t. \ e(o_i) = e(o_j) = e(o_k)
\end{aligned}
\label{eq:additive-env}
\end{equation}
This constraint reflects the intuition that the latent action from $o_i$ to $o_k$ 
can be decomposed as the sum of those from $o_i$ to $o_j$ and from $o_j$ to $o_k$. 

The additive structure on the latent action space can also be applied to the FDM, 
\begin{equation}
\begin{aligned}
F_{z_{jk}}\big(&F_{z_{ij}}(o_i)\big) = F_{z_{ij}+z_{jk}}(o_i), \\
    & \forall o_i, o_j, o_k, \ s.t. \ e(o_i) = e(o_j) = e(o_k)
\label{eq:fdm-comp-env}
\end{aligned}
\end{equation}
since on the LHS we have $F_{z_{jk}}\big(o_j\big) = o_k$ and RHS $F_{z_{ik}}(o_i)=o_k$.
The additive structure of the IDM form in Eq.~\eqref{eq:additive-env} and 
the FDM in Eq.~\eqref{eq:fdm-comp-env} will later be used to construct the auxiliary loss.

Although the motion composition nature of the rigid-body motion described by the Lie group $SE(3)$ is matrix-multiplicative, 
small inter-frame motions captured in common LAMs can be approximated as additive in a Euclidean space, 
under the Baker--Campbell--Hausdorff (BCH) approximation~\citep{barfoot2024state}. 
%
%
Concretely, the movement of a robot arm end-effector (say $p_t=[x,y,z,1]$) can be represented as a matrix $T_t \in \mathbb{R}^{4\times4}$, consisting of a rotation matrix $R\in\mathbb{R}^{3\times3}$ and translation vector $t\in\mathbb{R}^3$. Consecutive transformations are composed through matrix multiplication.
The translational aspect composes as the addition of xyz coordinates, shown below. 
While the the rotational composition is matrix-multiplicative, the BCH approximation models it as vector addition in the form of axis angle when the rotation is small.
\begin{align*}
p_3&= T_2 p_2 = T_2 (T_1 p_1)
\\
T_2 T_1 &=
\begin{bmatrix}
R_2 & t_2 \\
0 & 1
\end{bmatrix}
\begin{bmatrix}
R_1 & t_1 \\
0 & 1
\end{bmatrix}
=
\begin{bmatrix}
R_2 R_1 & R_2 t_1 + t_2 \\
0 & 1
\end{bmatrix}
\end{align*}

Further, the additive assumption provides two practical advantages. 
First, it makes structural priors easy to express as differentiable training objectives.
For example, this property can be turned into simple residuals (e.g., $z_{ik} - z_{ij} - z_{jk}$) 
with well-behaved gradients, without additional tricks (e.g., logarithm mapping) to handle non-linearity.
Second, it improves interpretability by calibrating the vector norm $\|z\|$ with displacement empirically.
Specifically, we observe that, in this additive space, the vector norm $\|z\|$ empirically tends
to correlate with physical motion magnitude over small time steps, yielding a transparent ``amount of motion'' signal. 

\begin{figure*}[htbp]
    \centering
    \includegraphics[width=0.8\textwidth]{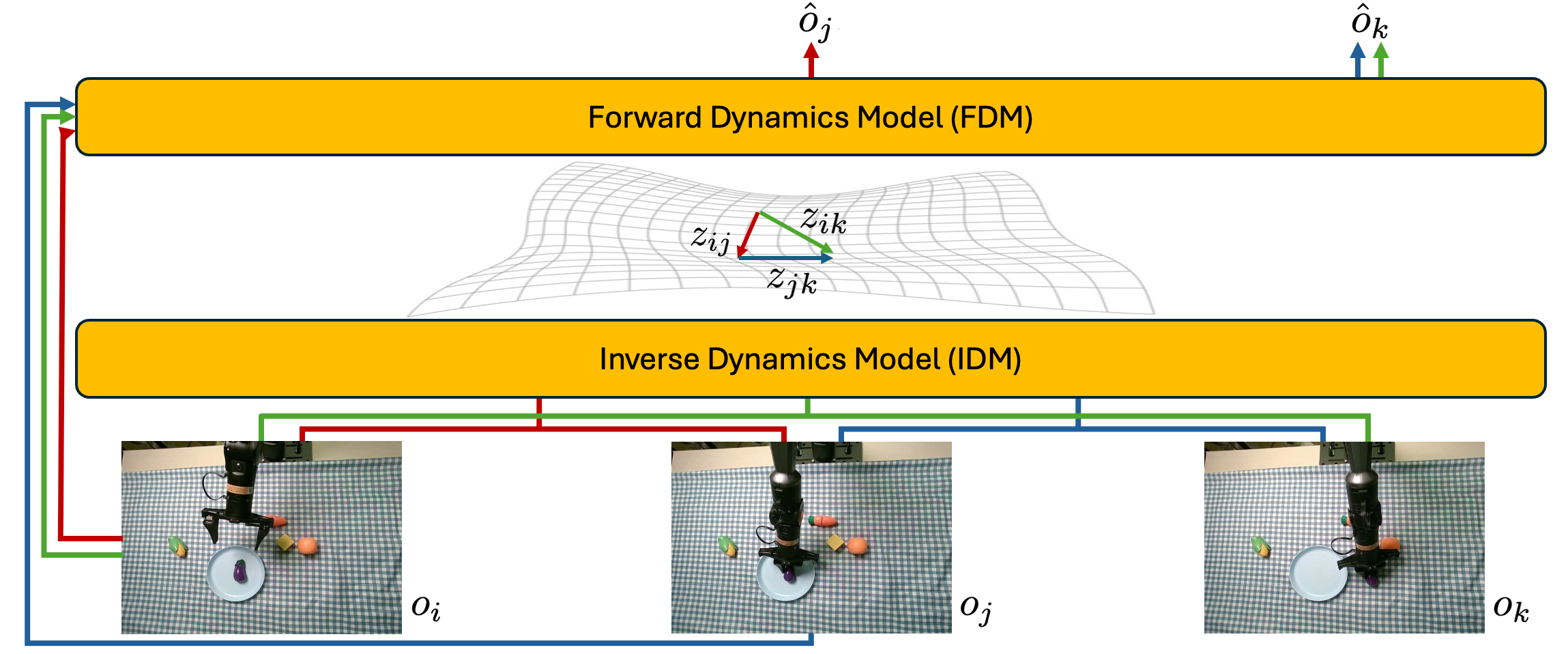}
    \caption{\textbf{Additively Compositional Latent Action Model~(AC-LAM).} 
    For triples $(o_i,o_j,o_k)$ from the same scene, scene-wise additivity encourages $z_{ik}\approx z_{ij}+z_{jk}$, which regularizes the latent action space on top of a standard IDM–FDM architecture. The red line denotes the $(i,j)$ mapping with IDM encoder $z_{ij} = f(o_i, o_j)$ and FDM decoder $\hat{o}_j=F_{z_{ij}}(o_i)$. The blue and green lines depict the corresponding mappings for $(j,k)$ and $(i,k)$, respectively. 
    }
    \label{fig:example}
\end{figure*}

\subsection{Analysis}
We now formally study what is implied by additivity in the latent action space.  
The below considers observations coming from the same scene $e(o)$.

Based on additivity defined in \eqref{eq:additive-env}, we can derive the following propositions about identity, inverse consistency, and cycle consistency.
\begin{proposition}[Identity]
\label{prop:identity}
It holds that $z_{ii}=0$.  
\end{proposition}
\begin{proof}
Apply additivity to $(i,i,k)$ to get $z_{ik}=z_{ii}+z_{ik}$.  
Cancel $z_{ik}$ on both sides to obtain $z_{ii}=0$.  
\end{proof}
\begin{proposition}[Inverse consistency]
\label{prop:inverse}
For any pair $(i,j)$ in the same scene, it holds that $z_{ji}=-z_{ij}$.  
\end{proposition}
\begin{proof}
Apply additivity to $(i,j,i)$ to get $z_{ii}=z_{ij}+z_{ji}$.  
Use Proposition~\ref{prop:identity} to substitute $z_{ii}=0$.  
Rearranging yields $z_{ji}=-z_{ij}$.  
\end{proof}

\begin{proposition}[Cycle consistency]
For any cycle $i_0\!\rightarrow\! i_1\!\rightarrow\!\cdots\!\rightarrow\! i_m=i_0$ within a scene, we have $\sum_{t=0}^{m-1} z_{i_t i_{t+1}}=0$.  
\end{proposition}
\begin{proof}
Repeatedly apply additivity along the path to obtain $\sum_{t=0}^{m-1} z_{i_t i_{t+1}}=z_{i_0 i_m}$.  
Since $i_m=i_0$, Proposition~\ref{prop:identity} gives $z_{i_0 i_m}=z_{i_0 i_0}=0$.  
\end{proof}

\paragraph{What does additivity bring to latent actions?}
We will show that enforcing additivity on the latent action discourages it to capture non-additive components.
Two typical examples for the non-additive components are scene-relative information (e.g., information about the background in the scene) and future leakage information (e.g., the target observation in the prediction of FDM).
This information would ideally not appear in the latent action as it is irrelevant to motion and control. 
However, in practice it does often appear in latent action models~\citep{li2025latbot, garrido2026}, despite bottlenecks designed to suppress it, as it provides a shortcut to optimize for the reconstruction loss~\eqref{eq:rec_loss}. 

First, we show that the additive structure discourages encoding static scene identifiers as a constant offset in $z$.  

\begin{proposition}[No scene-related bias]
In a fixed scene $s$,  
assume a decomposition $z_{ij}=\tilde{z}_{ij}+b_s$ for all $(i,j)$ in scene $s$, where $\tilde{z}_{ij}$ represents movement and $b_s\in\mathcal{Z}$ is constant w.r.t.\ $(i,j)$.  
If both $z$ and $\tilde{z}$ satisfy additivity within scene $s$, then $b_s=0$.  
\end{proposition}
\begin{proof}
From additivity $z_{ik}=z_{ij}+z_{jk}$,
substitute $z_{pq}=\tilde{z}_{pq}+b_s$ to get $\tilde{z}_{ik}+b_s=(\tilde{z}_{ij}+b_s)+(\tilde{z}_{jk}+b_s)$.  
Use additivity of $\tilde{z}$ to replace $\tilde{z}_{ij}+\tilde{z}_{jk}$ by $\tilde{z}_{ik}$.  
This yields $\tilde{z}_{ik}+b_s=\tilde{z}_{ik}+2b_s$, hence $b_s=0$.  
\end{proof}


Next, we show that additivity also suppresses another common pitfall in LAM where $z_{ij}$ ignores the motion and directly embeds the goal of FDM prediction $o_j$.  
\begin{proposition}[No future leakage]
In a fixed scene $s$,  
assume a decomposition $z_{ij}=\tilde{z}_{ij}+g_j$ for all $(i,j)$ in scene $s$, where $g_j\in\mathcal{Z}$ depends only on the goal index $j$.  
If both $z$ and $\tilde{z}$ satisfy additivity within scene $s$, then $g_j=0$ for all $j$ in scene $s$.  
\end{proposition}
\begin{proof}
From additivity $z_{ik}=z_{ij}+z_{jk}$,
substitute $z_{pq}=\tilde{z}_{pq}+g_q$ to get $\tilde{z}_{ik}+g_k=(\tilde{z}_{ij}+g_j)+(\tilde{z}_{jk}+g_k)$.  
Use additivity of $\tilde{z}$ to replace $\tilde{z}_{ij}+\tilde{z}_{jk}$ by $\tilde{z}_{ik}$.  
This yields $\tilde{z}_{ik}+g_k=\tilde{z}_{ik}+g_j+g_k$, hence $g_j=0$.  
\end{proof}

While above results rest on simplified assumptions—e.g., approximate linear separability between scene/goal terms and motion semantics—that may not hold in practice, they clarify how additive structure regularizes latent actions:
to encode additive components (which are typically related to rigid-body movement) while suppressing the non-additive components such as the scene-related and goal-only terms, reducing these common pitfalls for LAM training.


\subsection{Implementation}


Out of convenience, we treat each trajectory as a scene (by Def.~\ref{def_scene} the observations within a trajectory are reachable with the given sequence of actions).
Note this results in a finer division of scenes than strictly necessary.


To enforce the additive constraint on the latent action space there are two possible candidates. Firstly, one could consider applying it to the IDM in Eq.~\eqref{eq:additive-env}.
\begin{equation}
\mathcal{L}_{\mathrm{AC\text{-}IDM}}=\mathbb{E}_{i,j,k}\,\|f(o_i,o_k)-f(o_i,o_j)-f(o_j,o_k)\|_2^2.
\label{eq:ac-idm-impl}
\end{equation}
However, empirically we found that this IDM form led to instability during optimization (a trivial a solution collapses all latents to zero $f(o_i,o_j) = 0 \; \forall i, j$), requiring target stabilization techniques (e.g., EMA~\cite{grill2020byol,he2019momentum} or frozen targets~\cite{mnih2015humanlevel}).

The second option for enforcing additivity -- through the FDM in Eq.~\eqref{eq:fdm-comp-env} -- is more stable. 
\begin{equation}
\mathcal{L}_{\mathrm{AC\text{-}FDM}}=\mathbb{E}_{i,j,k}\,\|\,o_k-F(o_i, f(o_i,o_j)+f(o_j,o_k))\|_2^2.
\label{eq:ac-fdm-impl}
\end{equation}
Hence, the final AC-LAM objective is
\begin{equation}
\mathcal{L} \;=\; \mathcal{L}_{\mathrm{rec}} \;+\; \mathcal{L}_{\mathrm{reg}} \;+\; \lambda_{\mathrm{AC}}\,\mathcal{L}_{\mathrm{AC\text{-}FDM}}
\end{equation}
where $\mathcal{L}_{\mathrm{rec}}$ is from Eq.~\eqref{eq:rec_loss}, $\mathcal{L}_{\mathrm{reg}}$ comprises VQ-VAE codebook and commitment losses~\citep{ye2024latent,chen2025villax}, and $\lambda_{\mathrm{AC}}$ weights the AC-FDM term.

We apply AC constraints to post‑VQ continuous embeddings used as latent actions. To align with prior LAMs and enable fair comparison, we adopt a VQ‑VAE bottleneck \citep{van2017neural,ye2024latent,chen2025villax}, though AC is bottleneck‑agnostic in principle. Our implementation builds on the villa‑X design~\citep{chen2025villax} for its strong performance and includes a proprioceptive FDM for reconstructing robot state.

\paragraph{Latent Action Evaluation}
We evaluate AC-LAM by training a policy with latent actions as supervision and testing it in closed loop.  
We adopt the policy architecture in villa-X (see Appendix \ref{appendix:policy training}). 
Policies are trained with continuous post‑VQ latent embeddings from the latent action model as supervision, consistent with villa‑X. Downstream policy performance serves as a proxy for latent‑action quality.


\section{Experiments}
Our experiments aim to answer the following questions:

\begin{itemize}
    \item Q1. Does AC-LAM learn well-structured latent actions?
    \item Q2. Does AC-LAM outperform state-of-the-art LAMs such as those in LAPA, UniVLA, and Villa-X in downstream policy learning? 
    \item Q3. How do different design choices affect the performance of AC-LAM? 
\end{itemize}

\subsection{Experimental Setup}

\paragraph{Baselines} We compare AC-LAM against the following baseline latent action models:
\begin{itemize}
    \item \textbf{LAPA LAM}~\citep{ye2024latent} learns discrete latent actions via a VQ-VAE latent action tokenizer.
    \item \textbf{UniVLA LAM}~\citep{bu2025univlalearningacttaskcentric} learns task-centric latent actions through language conditioning to extract task-relevant dynamics.
    \item \textbf{Villa-X LAM}~\citep{chen2025villax} learns physically grounded latent actions with an extra proprioceptive state FDM. 
\end{itemize}


\paragraph{Implementation Details}
We build AC-LAM on the Villa-X LAM architecture and add additive-composition (AC) regularization to structure the latent action space.
The pretraining follows Villa-X the same datasets including OpenX and large-scale human video corpora. 
For scene-wise AC sampling, we draw triples \((i,j,k)\) from the same trajectory, bound the temporal span across \(i,j,k\) by \(\tau\) (robot: \(3\,\mathrm{s}\); human: \(2\,\mathrm{s}\)), and filter triples with large rotations on robot data
to ensure additivity is a reasonable approximation.
Additional architectural, training, and dataset details are provided in Appendix \ref{appendix:LAM training} and \ref{appendix:LAM datasets}.

\subsection{Does AC-LAM learn well-structured latent actions?}



We assess whether AC-LAM learns latent actions with better structure along four axes: (i) adherence to additive composition, (ii) alignment between latent-norm and true motion magnitude, (iii) emergence of identity and inverse elements, (iv) suppression of non‑compositional leakage (environment‑specific and goal‑only terms). We perform quantitative/qualitative analyses on Fractal~~\citep{brohan2022rt1}~(in-distribution), Bridge-V2~\citep{walke2023bridgedata}~(in-distribution) and LIBERO~\citep{liu2023libero}~(out-of-distribution). More details can be found in Appendix~\ref{appendix:latent_action_structure}. 



\paragraph{Adherence to additive composition} We quantify adherence to the scene‑wise additive composition prior with a normalized composition residual
\begin{equation}
\mathcal{L}_{\mathrm{Norm\text{-}AC}}=\frac{\mathbb{E}_{i,j,k}\,\|z_{ik}-z_{ij}-z_{jk}\|_2^2}{\mathbb{E}_{i,j}\,\|{z_{ij}\|_2^2}}.
\label{eq:Norm-AC}
\end{equation}
This scale‑invariant metric measures how closely latent actions compose additively within a scene; lower values indicate stronger adherence
(since Eq~\eqref{eq:additive-env} suggests $z_{ik}-z_{ij}-z_{jk}=0$).
As shown in Table~1, AC‑LAM achieves lower $\mathcal{L}_{\mathrm{Norm\text{-}AC}}$ than baselines, indicating stronger adherence to additive composition. 

\paragraph{Alignment between latent-norm and true motion magnitude} 

To assess displacement calibration, we conduct both quantitative and qualitative evaluations.


Quantitatively, we compute the Pearson correlation between the latent‑action norm $\|z\|$ (using continuous post‑VQ embeddings) and the norm of the change in proprioceptive states $|\Delta s|$. 
\begin{equation}
r\big(\|z\|,\,\|\Delta s\|\big)
= \frac{\operatorname{cov}\!\big(\|z\|,\,\|\Delta s\|\big)}{\sigma_{\|z\|}\,\sigma_{\|\Delta s\|}}
\end{equation}
As shown in Table 1, 
AC‑LAM demonstrates strong alignment between latent magnitude and physical motion, with broadly favorable results across datasets.


Qualitatively, we visualize the trajectory of the latent action norm $\|f(o_0,o_t)\|$ over time on real‑world tabletop manipulation in Figure~\ref{fig:latent-norm-trajectory-full} . Despite intermediate fluctuations, AC‑LAM tracks displacement magnitude most faithfully.
In contrast, LAPA LAM and UniVLA LAM show weak correspondence to displacement, and villa‑X LAM, while more correlated due to the proprio FDM design, remains systematically under‑calibrated relative to AC‑LAM.
The norm thus offers an interpretable proxy for the ``amount of motion" from the initial observation.


\begin{table*}[t]
\centering
\small
\caption{
Comparison of structured latent metrics on Fractal, Bridge and LIBERO for AC-LAM vs baselines. Lower $\mathcal{L}_{\mathrm{Norm\text{-}AC}}$/$\mathcal{L}_{\mathrm{Norm\text{-}Identity}}$/$\Delta_{\mathrm{inv}}$ indicate stronger additive/identity/inverse consistency; higher $r\big(\|z\|,\,\|\Delta s\|\big)$ indicate better displacement calibration. AC-LAM generally improves latent structure across datasets.
}
\begin{tabular}{lcccccc}
\toprule
\textbf{Dataset} & \textbf{Method} & \textbf{$\mathcal{L}_{\mathrm{Norm\text{-}AC}}$$\downarrow$} & \textbf{$r\big(\|z\|,\,\|\Delta s\|\big)$$\uparrow$} & $\mathcal{L}_{\mathrm{Norm\text{-}Identity}}$$\downarrow$ & $\Delta_{\mathrm{inv}}$$\downarrow$\\
\midrule
\multirow{4}{*}{Fractal}
    & AC-LAM                    & \textbf{0.086}             & \textbf{0.489}             & \textbf{0.151}                    & \textbf{0.156}                  \\
    & \makecell[c]{Villa-X LAM}     & 0.374             & 0.289             & 0.358                   & 0.874                  \\
    & \makecell[c]{UniVLA LAM}      & 0.960             & 0.325             & 0.478                     & 1.536                    \\
    & \makecell[c]{LAPA LAM}        & 0.845             & 0.024             & 0.862                    & 1.773                    \\
\midrule
\multirow{4}{*}{Bridge}
    & AC-LAM                    & \textbf{0.135}             &  \textbf{0.408}            & \textbf{0.020}                     & \textbf{0.267}                   \\
    & \makecell[c]{Villa-X LAM}     & 0.386             &  0.225            & 0.326                   & 0.875                   \\
    & \makecell[c]{UniVLA LAM}      & 0.953             &  0.276            & 0.517                  & 1.465                \\
    & \makecell[c]{LAPA LAM}        & 0.875             & -0.074            & 0.914                & 1.787                 \\
\midrule
\multirow{4}{*}{LIBERO}
    & AC-LAM                    & \textbf{0.089}             & 0.277             & 1.338            & \textbf{0.206}                     \\
    & \makecell[c]{Villa-X LAM}     & 0.491             & 0.074             & \textbf{0.516}                     & 0.955              \\
    & \makecell[c]{UniVLA LAM}      & 1.005             & \textbf{0.647}             & 0.517                  & 1.471               \\
    & \makecell[c]{LAPA LAM}        & 0.905             & 0.180             & 0.940                  & 1.808                \\
\bottomrule
\label{table:analysis}
\end{tabular}
\end{table*}

\paragraph{Emergence of identity and inverse elements} 


The scene‑wise additive composition prior induces identity and inverse structure in the latent action space. We empirically quantify these properties using continuous post‑VQ latents. To test how accurately the identity is learned, we report:
\begin{equation}
\mathcal{L}_{\mathrm{Norm\text{-}Identity}}=\frac{\mathbb{E}_{i}\,\|z_{ii}\|}{\mathbb{E}_{i,j}\,\|z_{ij}\|}
\end{equation}
where values closer to zero indicate stronger emergence of the identity property (as in Proposition~\ref{prop:identity}).
We evaluate the normalized norm for the sum of inverse elements
\begin{equation}
\Delta_{\mathrm{inv}}
= \frac{\mathbb{E}_{i,j}\,\|z_{ij} + z_{ji}\|}{\mathbb{E}_{i,j}\,\|z_{ij}\|}
\label{eq:mag-ratio}
\end{equation}
where values closer to zero indicate stronger emergence of the inverse property (as in Proposition~\ref{prop:inverse}).

Across datasets, AC‑LAM generally trends closer to the identity and inverse ideals, suggesting a more structured and interpretable latent‑action space.


\paragraph{Suppression of non‑compositional leakage} 

The scene‑wise additive composition prior is intended to suppress non‑compositional signals (environment identifiers and future information leakage). To quantify environment‑specific leakage, we train an XGBoost classifier~\citep{chen2016xgboost} to predict the robot dataset ID from continuous post‑VQ latent actions, treating the ID as a proxy for environment features. 
Lower probe accuracy $\mathrm{Acc}_{\mathrm{env}}^{\mathrm{mlp}}$ indicates less environment information in the latents and, consequently, stronger cross‑environment generalization.



Quantifying future information leakage is more challenging because future goals legitimately correlate with motion semantics. We therefore do not measure leakage directly; instead, we assess it indirectly via proxies—primarily the additivity residual across distinct goals within the same scene (Eq.~\ref{eq:Norm-AC}), which quantifies practical adherence to additive composition. Consequently, future information leakage is evaluated through this composition‑adherence metric. 


Across these evaluations, AC‑LAM consistently yields lower environment‑ID probe accuracy~(Table \ref{tab:env_probe}) and stronger proxy signals~(Table \ref{table:analysis}) than baselines, indicating better suppression of non‑compositional leakage and improved generalization.

\begin{table}[t!]
  \caption{Evaluation results on probe accuracy~(\texttt{$\mathrm{Acc}_{\mathrm{env}}^{\mathrm{mlp}}$}) on four environments~(fractral, bridge, kuka, droid) with different LAMs. Lower probe accuracy indicates less environment leakage.
}
  \label{tab:env_probe}
  \centering
    \small
    \begin{tabular}{lcccc}
        \toprule
        \textbf{LAM} & \makecell{\small{AC-LAM}} & \makecell{\small{Villa-X}} & \makecell{\small{UniVLA}} & \makecell{\small{LAPA}} \\ 
        \midrule
        \texttt{$\mathrm{Acc}_{\mathrm{env}}^{\mathrm{mlp}}$$\downarrow$} & \textbf{50.2\%} & 63.0\% & 80.2\% & 52.0\% \\
        \bottomrule
  \end{tabular}
\end{table}

\begin{figure*}[htbp]
    \centering
    \includegraphics[width=0.75\textwidth]{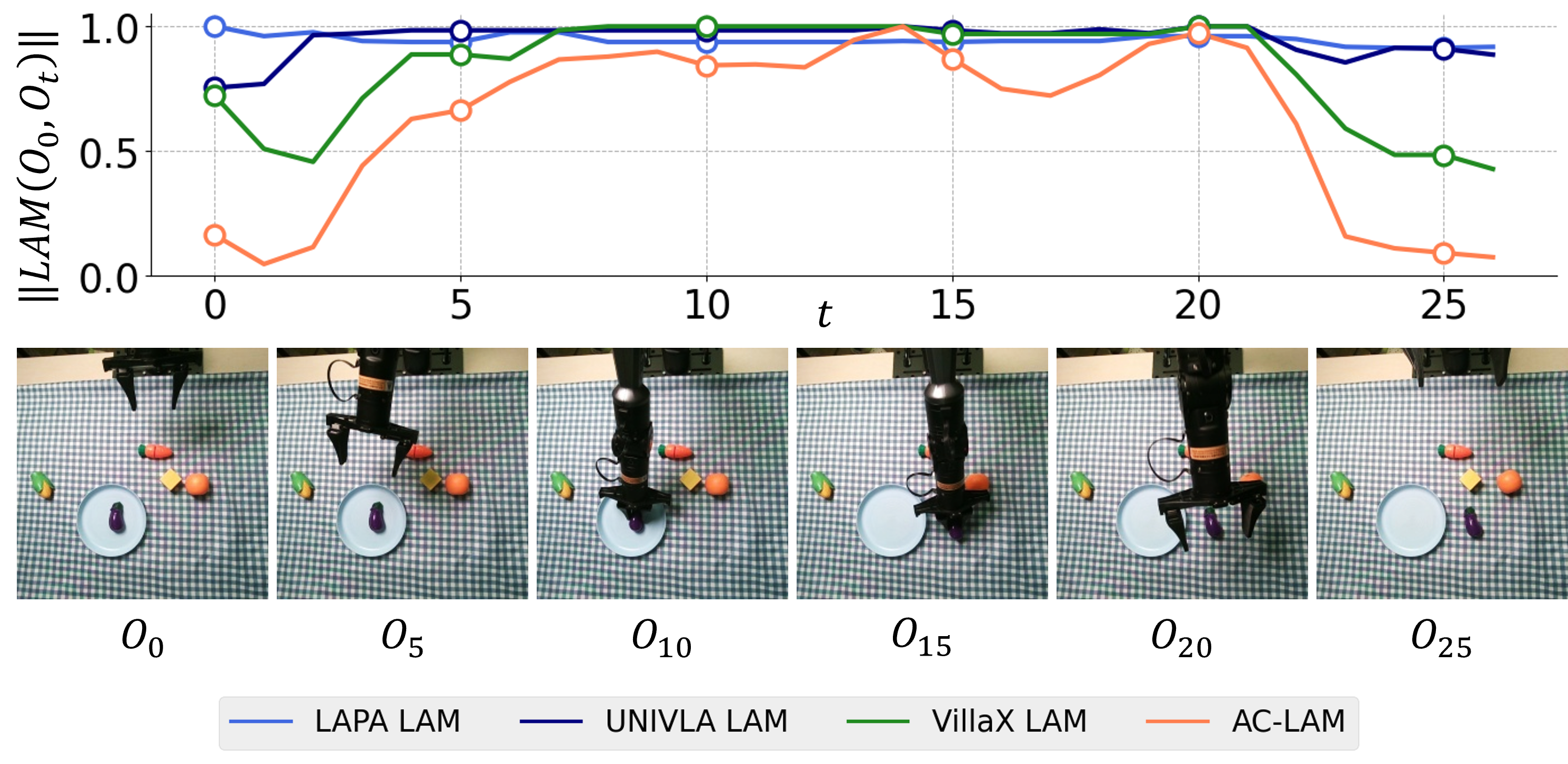}
    \caption{
    Trajectory of the latent action norm $||f(o_0,o_t)||$ in real-world tabletop manipulation, with latent actions generated by LAPA LAM, UniVLA LAM, Villa-X LAM and AC-LAM. AC‑LAM yields the most displacement‑calibrated latents, aligning with motion magnitude.
    }
    \label{fig:latent-norm-trajectory-full}
\end{figure*}

\subsection{Does AC-LAM improve downstream policy learning performance?}
We further evaluate AC-LAM's ability to provide effective supervision signals for downstream policy learning. 

\paragraph{Benchmarks}
We assess both vision–semantic generalization in simulation and accurate control on real-world tabletop manipulation~(see more details in Appendix~\ref{appendix:benchmarks}). 

\begin{itemize}
\item \textbf{Emoji Table-Top (GrinningFace)}~\citep{zhang2025vlaseffectivelyinheritvlms}: A diagnostic simulation benchmark for vision–semantic generalization in embodied control. Each episode uses the instruction “Pick the cube and place it on [desc.]”, where [desc.] is the language description of the target emoji. Three emoji cards are placed on the tabletop; success requires grasping the cube and placing it on the correct target card. 
We evaluate under three protocols supported by the benchmark: ID (in-distribution combinations and order), Train (novel combinations of training-set emojis), and Val (held-out validation emojis; out-of-distribution).
    \item \textbf{Real-World Tabletop Manipulation} A physical evaluation of accurate control and robustness using an AgileX Robotics Piper arm (7-DoF), focused on diverse pick tasks across varied objects and backgrounds. The dataset comprises 170 teleoperated trajectories collected under varied tabletop settings—including different tablecloth textures/colors, object layouts, and object positions—to increase scene diversity and support robustness evaluation. Performance is assessed under three regimes: in-distribution~(ID) scenes, out-of-distribution distractors (OOD-D, novel or repositioned non-target objects), and out-of-distribution backgrounds (OOD-B, changes to tabletop/background appearance).
\end{itemize}

\paragraph{Policy Training Setup}

We evaluate policies in the cross-dataset generalization setting, augmenting training data with Bridge-V2 to assess knowledge transfer.
We use LAM to derive latent-action labels $z$ forming robot tuples $(s, z, a)$. The policy follows the Villa-X architecture~\citep{chen2025villax} and is trained end-to-end with joint supervision from both latent actions and robot actions. 


\paragraph{Baselines}
We compare policies trained with different LAMs (Section~3.1) to assess each LAM's ability to provide supervision that enables downstream policy learning. We also include a baseline trained from scratch using only action-labeled data, based on the $\pi_0$~\citep{black2024pi0} architecture (i.e., without latent-action supervision). For fairness, we align architecture and dataset across settings: Villa-X extends $\pi_0$ with a latent-action decoder, so $\pi_0$~(or w/o LAM) serves as the corresponding variant without latent-action supervision, isolating the benefit of latent labels.

\begin{table*}[t!]
  \caption{Policy results on Emoji Table-Top (GrinningFace) across ID, Train, and Val splits.
  We report Success (S) (placement on the correct target), Any-Success (S/A) (placement on any target), and the Recognition Ratio (R=S/(S/A)), which approximates the model's target recognition accuracy.
}
  \label{tab:Emoji_Table-Top}
  \centering
    \small
    \begin{tabular}{lccccccccc}
        \toprule
        \multirow{2}{*}{\textbf{Method}} & \multicolumn{3}{c}{\textbf{ID}} & \multicolumn{3}{c}{\textbf{Train}} & \multicolumn{3}{c}{\textbf{Val}}  \\
        \cmidrule(lr){2-4}  \cmidrule(lr){5-7} \cmidrule(lr){8-10} 
         & S & S/A & R & S & S/A & R & S & S/A & R\\ 
        \midrule
        \texttt{w/o LAM} & 22 & 35 & 0.63 & 26 & 44 & 0.59 & 19 & 38 & 0.50 \\
        \midrule
        LAPA LAM & 19 & 33 & 0.58 & 7 & 26 & 0.27 & 9 & 34 & 0.26 \\
        UniVLA LAM & 38 & 55 & 0.69 & 29 & \textbf{58} & 0.5 & 28 & 57 & 0.49 \\
        Villa-X LAM & 42 & 53 & 0.79 & 25 & 52 & 0.48 & 31 & 58 & 0.53 \\
        \midrule
        AC-LAM & \textbf{55} & \textbf{61} & \textbf{0.90} & \textbf{42} & 56 & \textbf{0.75} & \textbf{41} & \textbf{60} & \textbf{0.68} \\
        \bottomrule
  \end{tabular}
\end{table*}

\begin{table}[t!]
  \caption{Evaluation results on Real-World Tabletop Manipulation for policy learning with different LAMs. 
}
  \label{tab:real_robot}
  \centering
    \small
    \begin{tabular}{lccc}
        \toprule
        \textbf{Method} & \makecell{\small{ID}} & \makecell{\small{OOD-D}} & \makecell{\small{OOD-B}} \\ 
        \midrule
        \texttt{w/o LAM} & 33.3 & 26.7 & 6.7 \\
        \midrule
        LAPA LAM & 13.3 & 6.7 & 0 \\
        UniVLA LAM & 33.3 & 26.7 & 26.7 \\
        Villa-X LAM & 40 & 20 & 26.7\\
        \midrule
        AC-LAM & \textbf{60} & \textbf{53.3} & \textbf{33.3} \\
        \bottomrule
  \end{tabular}
\end{table}



\paragraph{Results and Analysis}
Tables~\ref{tab:Emoji_Table-Top} (simulation) and~\ref{tab:real_robot} (real robot) report task success across protocols and tasks. Trends are consistent across simulation and real robot: adding scene-wise additive-composition (AC) constraints to latent action learning yields marked improvements over the no‑AC baseline in average task success, with gains observed on most tasks. On Emoji Table‑Top, AC‑LAM achieves higher success under ID/Train/Val. In Table~\ref{tab:Emoji_Table-Top}, we report S (grasp the cube and place it on the correct target card), S/A (grasp the cube and place it on any card), and $R = \frac{S}{S/A}$, which approximates target‑card recognition. AC‑LAM notably improves R, contributing to the overall success rate.
On the real‑robot suite, AC‑LAM consistently outperforms LAPA, UniVLA, villa‑X, and $\pi_0$~(or w/o LAM), indicating robust generalization and more effective supervision under distribution shift in real‑world robot settings. 
Under OOD background shifts, AC‑LAM typically preserves pick success while the w/o LAM baseline~($\pi_0$) and LAPA LAM fails to grasp consistently. We attribute these gains to motion‑specific, displacement‑calibrated latents induced by the additive‑composition prior, which provide stronger supervision in downstream policy learning.






\subsection{How do different design choices affect AC-LAM performance?}



We study design choices for instantiating the scene-wise additive composition prior, targeting (i) practical adherence to additive composition, (ii) improved calibration on motion magnitude, and (iii) greater training stability. We evaluate above structured-latent metrics and optimization stability, and conduct ablations on AC-LAM trained on a reduced dataset comprising Bridge-V2 and Sth-Sth-V2~\citep{goyal2017something}.


\paragraph{Design Factors}
Our default AC-LAM uses the FDM-form loss with constraints on post-VQ embeddings. We ablate:
\begin{itemize}
    \item \textbf{AC loss form}: $\mathcal{L}_{\mathrm{AC\text{-}IDM}}$ (additivity in inverse-dynamics latents; Eq.~\ref{eq:ac-idm-impl}) vs.\ $\mathcal{L}_{\mathrm{AC\text{-}FDM}}$ (decoding from summed latents; Eq.~\ref{eq:ac-fdm-impl}).
    \item \textbf{Placement in VQ-VAE}: pre-VQ (continuous encoder latents) vs.\ post-VQ (codebook embeddings).
    \item \textbf{Stop-gradient (sg) for $\mathcal{L}_{\mathrm{AC\text{-}IDM}}$}: none, sg on $z_{ik}$, or sg on $(z_{ij}+z_{jk})$.

\end{itemize}

\paragraph{Findings} 

Enforcing the scene-wise additive prior via decoder-side $\mathcal{L}_{\mathrm{AC\text{-}FDM}}$ is generally more stable than $\mathcal{L}_{\mathrm{AC\text{-}IDM}}$, as decoding from the sum of latents regularizes through the observation space and reduces optimization shocks. 
We observe a trade-off in where AC is applied: \emph{pre-VQ} improves displacement calibration ($r(\|z\|,\|\Delta s\|)=\mathbf{0.476}$) but weakens additive consistency ($\mathcal{L}_{\mathrm{Norm\text{-}AC}}=0.456$), while \emph{post-VQ} strengthens additive consistency ($\mathcal{L}_{\mathrm{Norm\text{-}AC}}=\mathbf{0.102}$) with stable training (and moderate $r=0.256$).
For $\mathcal{L}_{\mathrm{AC\text{-}IDM}}$, stop-gradient placement is critical: without stop-gradient the training tends to collapse; stopping gradients on $(z_{ij}+z_{jk})$ often drives the latent norm to blow up; stopping on $z_{ik}$ is the most stable of the IDM variants, yet it still underperforms $\mathcal{L}_{\mathrm{AC\text{-}FDM}}$.
In practice, both pre-VQ and post-VQ have merits; following prior LAM/policy setups that consume post-VQ codebook embeddings~\citep{chen2025villax}, we default to applying $\mathcal{L}_{\mathrm{AC\text{-}FDM}}$ on post-VQ continuous embeddings.

\begin{table}[t!]
  \caption{Ablation on different design choices for AC-LAM, evaluated on Bridge.
}
  \label{tab:ablation}
  \centering
    \small
    \begin{tabular}{lccc}
        \toprule
        \textbf{Design} & \textbf{$\mathcal{L}_{\mathrm{Norm\text{-}AC}}$$\downarrow$}  & \textbf{$r\big(\|z\|,\,\|\Delta s\|\big)$$\uparrow$}& Stability \\ 
        \midrule
        Default                     & \textbf{0.102} & 0.256            & stable    \\  \midrule
        pre-VQ                      & 0.456 & \textbf{0.476}   & stable    \\  \midrule
        IDM(no sg)                  & -          & -            & collapse  \\
        sg on $z_{ik}$         & 0.141          & 0.098            & stable    \\
        sg on $z_{ij}+z_{jk}$  & -              & -                & explode   \\
        \bottomrule
  \end{tabular}
\end{table}





\section{Conclusion and Future Work}



We introduced AC-LAM, a latent action learning framework that imposes a scene-wise additive-composition prior ($ z_{ik} = z_{ij} + z_{jk}$) aligned with short-horizon motion semantics. Our analysis and experiments show that AC-LAM yields more interpretable, motion-specific, and displacement-calibrated latents, suppresses non-compositional leakage, and provides stronger supervision for downstream policy learning. These results establish a practical and principled foundation for structured latent action learning from video.

Directions for future work include stronger scene identification to form cross-trajectory triples—via scene labels or unsupervised scene recognition/clustering. 
We have demonstrated that additive composition is a promising direction in latent action learning; advancing scene conditioning might further align structured latents with physical motion priors.

\section{Impact Statement}
This paper presents work whose goal is to advance the field of machine learning. There are many potential societal consequences of our work, none of which we feel must be specifically highlighted here.

\bibliography{latent_action}

@article{schmidt2023learning,
  title={Learning to act without actions},
  author={Schmidt, Dominik and Jiang, Minqi},
  journal={arXiv preprint arXiv:2312.10812},
  year={2023}
}

@inproceedings{bruce2024genie,
  title={Genie: Generative interactive environments},
  author={Bruce, Jake and Dennis, Michael D and Edwards, Ashley and Parker-Holder, Jack and Shi, Yuge and Hughes, Edward and Lai, Matthew and Mavalankar, Aditi and Steigerwald, Richie and Apps, Chris and others},
  booktitle={Forty-first International Conference on Machine Learning},
  year={2024}
}

@article{ye2024latent,
  title   = {Latent Action Pretraining from Videos},
  author  = {Seonghyeon Ye and Joel Jang and Byeongguk Jeon and Sejune Joo and Jianwei Yang and Baolin Peng and Ajay Mandlekar and Reuben Tan and Yu-Wei Chao and Bill Yuchen Lin and Lars Liden and Kimin Lee and Jianfeng Gao and Luke Zettlemoyer and Dieter Fox and Minjoon Seo},
  year    = {2024},
  journal = {arXiv preprint arXiv: 2410.11758}
}

@article{chen2024igor,
  title={IGOR: Image-GOal Representations are the Atomic Control Units for Foundation Models in Embodied AI},
  author={Chen, Xiaoyu and Guo, Junliang and He, Tianyu and Zhang, Chuheng and Zhang, Pushi and Yang, Derek Cathera and Zhao, Li and Bian, Jiang},
  journal={arXiv preprint arXiv:2411.00785},
  year={2024}
}

@misc{bu2025univlalearningacttaskcentric,
      title={UniVLA: Learning to Act Anywhere with Task-centric Latent Actions}, 
      author={Qingwen Bu and Yanting Yang and Jisong Cai and Shenyuan Gao and Guanghui Ren and Maoqing Yao and Ping Luo and Hongyang Li},
      year={2025},
      eprint={2505.06111},
      archivePrefix={arXiv},
      primaryClass={cs.RO},
      url={https://arxiv.org/abs/2505.06111}, 
}

@inproceedings{walke2023bridgedata,
    title={BridgeData V2: A Dataset for Robot Learning at Scale},
    author={Walke, Homer and Black, Kevin and Lee, Abraham and Kim, Moo Jin and Du, Max and Zheng, Chongyi and Zhao, Tony and Hansen-Estruch, Philippe and Vuong, Quan and He, Andre and Myers, Vivek and Fang, Kuan and Finn, Chelsea and Levine, Sergey},
    booktitle={Conference on Robot Learning (CoRL)},
    year={2023}
}

@article{chen2024moto,
  title   = {Moto: Latent Motion Token as the Bridging Language for Robot Manipulation},
  author  = {Yi Chen and Yuying Ge and Yizhuo Li and Yixiao Ge and Mingyu Ding and Ying Shan and Xihui Liu},
  year    = {2024},
  journal = {arXiv preprint arXiv: 2412.04445}
}

@article{black2024pi0,
  title   = {$\pi_0$: A Vision-Language-Action Flow Model for General Robot Control},
  author  = {Kevin Black and Noah Brown and Danny Driess and Adnan Esmail and Michael Equi and Chelsea Finn and Niccolo Fusai and Lachy Groom and Karol Hausman and Brian Ichter and Szymon Jakubczak and Tim Jones and Liyiming Ke and Sergey Levine and Adrian Li-Bell and Mohith Mothukuri and Suraj Nair and Karl Pertsch and Lucy Xiaoyang Shi and James Tanner and Quan Vuong and Anna Walling and Haohuan Wang and Ury Zhilinsky},
  year    = {2024},
  journal = {arXiv preprint arXiv: 2410.24164}
}

@article{agibot-world-contributors2025agibot,
  title   = {AgiBot World Colosseo: A Large-scale Manipulation Platform for Scalable and Intelligent Embodied Systems},
  author  = {Qingwen Bu and Jisong Cai and Li Chen and Xiuqi Cui and Yan Ding and Siyuan Feng and Shenyuan Gao and Xindong He and Xu Huang and Shu Jiang and Yuxin Jiang and Cheng Jing and Hongyang Li and Jialu Li and Chiming Liu and Yi Liu and Yuxiang Lu and Jianlan Luo and Ping Luo and Yao Mu and Yuehan Niu and Yixuan Pan and Jiangmiao Pang and Yu Qiao and Guanghui Ren and Cheng Ruan and Jiaqi Shan and Yongjian Shen and Chengshi Shi and Mingkang Shi and Modi Shi and Chonghao Sima and Jianheng Song and Huijie Wang and Wenhao Wang and Dafeng Wei and Chengen Xie and Guo Xu and Junchi Yan and Cunbiao Yang and Lei Yang and Shukai Yang and Maoqing Yao and Jia Zeng and Chi Zhang and Qinglin Zhang and Bin Zhao and Chengyue Zhao and Jiaqi Zhao and Jianchao Zhu, AgiBot-World-Contributors},
  year    = {2025},
  journal = {arXiv preprint arXiv: 2503.06669}
}

@article{liu2023libero,
  title={LIBERO: Benchmarking Knowledge Transfer for Lifelong Robot Learning},
  author={Liu, Bo and Zhu, Yifeng and Gao, Chongkai and Feng, Yihao and Liu, Qiang and Zhu, Yuke and Stone, Peter},
  journal={arXiv preprint arXiv:2306.03310},
  year={2023}
}

@misc{open_x_embodiment_rt_x_2023,
title={Open {X-E}mbodiment: Robotic Learning Datasets and {RT-X} Models},
author = {Open X-Embodiment Collaboration and Abby O'Neill and Abdul Rehman and Abhiram Maddukuri and Abhishek Gupta and Abhishek Padalkar and Abraham Lee and Acorn Pooley and Agrim Gupta and Ajay Mandlekar and Ajinkya Jain and Albert Tung and Alex Bewley and Alex Herzog and Alex Irpan and Alexander Khazatsky and Anant Rai and Anchit Gupta and Andrew Wang and Andrey Kolobov and Anikait Singh and Animesh Garg and Aniruddha Kembhavi and Annie Xie and Anthony Brohan and Antonin Raffin and Archit Sharma and Arefeh Yavary and Arhan Jain and Ashwin Balakrishna and Ayzaan Wahid and Ben Burgess-Limerick and Beomjoon Kim and Bernhard Schölkopf and Blake Wulfe and Brian Ichter and Cewu Lu and Charles Xu and Charlotte Le and Chelsea Finn and Chen Wang and Chenfeng Xu and Cheng Chi and Chenguang Huang and Christine Chan and Christopher Agia and Chuer Pan and Chuyuan Fu and Coline Devin and Danfei Xu and Daniel Morton and Danny Driess and Daphne Chen and Deepak Pathak and Dhruv Shah and Dieter Büchler and Dinesh Jayaraman and Dmitry Kalashnikov and Dorsa Sadigh and Edward Johns and Ethan Foster and Fangchen Liu and Federico Ceola and Fei Xia and Feiyu Zhao and Felipe Vieira Frujeri and Freek Stulp and Gaoyue Zhou and Gaurav S. Sukhatme and Gautam Salhotra and Ge Yan and Gilbert Feng and Giulio Schiavi and Glen Berseth and Gregory Kahn and Guangwen Yang and Guanzhi Wang and Hao Su and Hao-Shu Fang and Haochen Shi and Henghui Bao and Heni Ben Amor and Henrik I Christensen and Hiroki Furuta and Homer Walke and Hongjie Fang and Huy Ha and Igor Mordatch and Ilija Radosavovic and Isabel Leal and Jacky Liang and Jad Abou-Chakra and Jaehyung Kim and Jaimyn Drake and Jan Peters and Jan Schneider and Jasmine Hsu and Jeannette Bohg and Jeffrey Bingham and Jeffrey Wu and Jensen Gao and Jiaheng Hu and Jiajun Wu and Jialin Wu and Jiankai Sun and Jianlan Luo and Jiayuan Gu and Jie Tan and Jihoon Oh and Jimmy Wu and Jingpei Lu and Jingyun Yang and Jitendra Malik and João Silvério and Joey Hejna and Jonathan Booher and Jonathan Tompson and Jonathan Yang and Jordi Salvador and Joseph J. Lim and Junhyek Han and Kaiyuan Wang and Kanishka Rao and Karl Pertsch and Karol Hausman and Keegan Go and Keerthana Gopalakrishnan and Ken Goldberg and Kendra Byrne and Kenneth Oslund and Kento Kawaharazuka and Kevin Black and Kevin Lin and Kevin Zhang and Kiana Ehsani and Kiran Lekkala and Kirsty Ellis and Krishan Rana and Krishnan Srinivasan and Kuan Fang and Kunal Pratap Singh and Kuo-Hao Zeng and Kyle Hatch and Kyle Hsu and Laurent Itti and Lawrence Yunliang Chen and Lerrel Pinto and Li Fei-Fei and Liam Tan and Linxi "Jim" Fan and Lionel Ott and Lisa Lee and Luca Weihs and Magnum Chen and Marion Lepert and Marius Memmel and Masayoshi Tomizuka and Masha Itkina and Mateo Guaman Castro and Max Spero and Maximilian Du and Michael Ahn and Michael C. Yip and Mingtong Zhang and Mingyu Ding and Minho Heo and Mohan Kumar Srirama and Mohit Sharma and Moo Jin Kim and Naoaki Kanazawa and Nicklas Hansen and Nicolas Heess and Nikhil J Joshi and Niko Suenderhauf and Ning Liu and Norman Di Palo and Nur Muhammad Mahi Shafiullah and Oier Mees and Oliver Kroemer and Osbert Bastani and Pannag R Sanketi and Patrick "Tree" Miller and Patrick Yin and Paul Wohlhart and Peng Xu and Peter David Fagan and Peter Mitrano and Pierre Sermanet and Pieter Abbeel and Priya Sundaresan and Qiuyu Chen and Quan Vuong and Rafael Rafailov and Ran Tian and Ria Doshi and Roberto Mart{'i}n-Mart{'i}n and Rohan Baijal and Rosario Scalise and Rose Hendrix and Roy Lin and Runjia Qian and Ruohan Zhang and Russell Mendonca and Rutav Shah and Ryan Hoque and Ryan Julian and Samuel Bustamante and Sean Kirmani and Sergey Levine and Shan Lin and Sherry Moore and Shikhar Bahl and Shivin Dass and Shubham Sonawani and Shuran Song and Sichun Xu and Siddhant Haldar and Siddharth Karamcheti and Simeon Adebola and Simon Guist and Soroush Nasiriany and Stefan Schaal and Stefan Welker and Stephen Tian and Subramanian Ramamoorthy and Sudeep Dasari and Suneel Belkhale and Sungjae Park and Suraj Nair and Suvir Mirchandani and Takayuki Osa and Tanmay Gupta and Tatsuya Harada and Tatsuya Matsushima and Ted Xiao and Thomas Kollar and Tianhe Yu and Tianli Ding and Todor Davchev and Tony Z. Zhao and Travis Armstrong and Trevor Darrell and Trinity Chung and Vidhi Jain and Vincent Vanhoucke and Wei Zhan and Wenxuan Zhou and Wolfram Burgard and Xi Chen and Xiangyu Chen and Xiaolong Wang and Xinghao Zhu and Xinyang Geng and Xiyuan Liu and Xu Liangwei and Xuanlin Li and Yansong Pang and Yao Lu and Yecheng Jason Ma and Yejin Kim and Yevgen Chebotar and Yifan Zhou and Yifeng Zhu and Yilin Wu and Ying Xu and Yixuan Wang and Yonatan Bisk and Yongqiang Dou and Yoonyoung Cho and Youngwoon Lee and Yuchen Cui and Yue Cao and Yueh-Hua Wu and Yujin Tang and Yuke Zhu and Yunchu Zhang and Yunfan Jiang and Yunshuang Li and Yunzhu Li and Yusuke Iwasawa and Yutaka Matsuo and Zehan Ma and Zhuo Xu and Zichen Jeff Cui and Zichen Zhang and Zipeng Fu and Zipeng Lin},
howpublished  = {\url{https://arxiv.org/abs/2310.08864}},
year = {2023},
}

@inproceedings{grauman2022ego4d,
  title={Ego4d: Around the world in 3,000 hours of egocentric video},
  author={Grauman, Kristen and Westbury, Andrew and Byrne, Eugene and Chavis, Zachary and Furnari, Antonino and Girdhar, Rohit and Hamburger, Jackson and Jiang, Hao and Liu, Miao and Liu, Xingyu and others},
  booktitle={Proceedings of the IEEE/CVF Conference on Computer Vision and Pattern Recognition},
  pages={18995--19012},
  year={2022}
}

@article{damen2020epic,
  title={The epic-kitchens dataset: Collection, challenges and baselines},
  author={Damen, Dima and Doughty, Hazel and Farinella, Giovanni Maria and Fidler, Sanja and Furnari, Antonino and Kazakos, Evangelos and Moltisanti, Davide and Munro, Jonathan and Perrett, Toby and Price, Will and others},
  journal={IEEE Transactions on Pattern Analysis and Machine Intelligence},
  volume={43},
  number={11},
  pages={4125--4141},
  year={2020},
  publisher={IEEE}
}

@inproceedings{li2018eye,
  title={In the eye of beholder: Joint learning of gaze and actions in first person video},
  author={Li, Yin and Liu, Miao and Rehg, James M},
  booktitle={Proceedings of the European conference on computer vision (ECCV)},
  pages={619--635},
  year={2018}
}

@article{kim2024openvla,
  title={OpenVLA: An Open-Source Vision-Language-Action Model},
  author={Kim, Moo Jin and Pertsch, Karl and Karamcheti, Siddharth and Xiao, Ted and Balakrishna, Ashwin and Nair, Suraj and Rafailov, Rafael and Foster, Ethan and Lam, Grace and Sanketi, Pannag and others},
  journal={arXiv preprint arXiv:2406.09246},
  year={2024}
}

@article{khazatsky2024droid,
    title   = {DROID: A Large-Scale In-The-Wild Robot Manipulation Dataset},
    author  = {Alexander Khazatsky and Karl Pertsch and Suraj Nair and Ashwin Balakrishna and Sudeep Dasari and Siddharth Karamcheti and Soroush Nasiriany and Mohan Kumar Srirama and Lawrence Yunliang Chen and Kirsty Ellis and Peter David Fagan and Joey Hejna and Masha Itkina and Marion Lepert and Yecheng Jason Ma and Patrick Tree Miller and Jimmy Wu and Suneel Belkhale and Shivin Dass and Huy Ha and Arhan Jain and Abraham Lee and Youngwoon Lee and Marius Memmel and Sungjae Park and Ilija Radosavovic and Kaiyuan Wang and Albert Zhan and Kevin Black and Cheng Chi and Kyle Beltran Hatch and Shan Lin and Jingpei Lu and Jean Mercat and Abdul Rehman and Pannag R Sanketi and Archit Sharma and Cody Simpson and Quan Vuong and Homer Rich Walke and Blake Wulfe and Ted Xiao and Jonathan Heewon Yang and Arefeh Yavary and Tony Z. Zhao and Christopher Agia and Rohan Baijal and Mateo Guaman Castro and Daphne Chen and Qiuyu Chen and Trinity Chung and Jaimyn Drake and Ethan Paul Foster and Jensen Gao and David Antonio Herrera and Minho Heo and Kyle Hsu and Jiaheng Hu and Donovon Jackson and Charlotte Le and Yunshuang Li and Kevin Lin and Roy Lin and Zehan Ma and Abhiram Maddukuri and Suvir Mirchandani and Daniel Morton and Tony Nguyen and Abigail O'Neill and Rosario Scalise and Derick Seale and Victor Son and Stephen Tian and Emi Tran and Andrew E. Wang and Yilin Wu and Annie Xie and Jingyun Yang and Patrick Yin and Yunchu Zhang and Osbert Bastani and Glen Berseth and Jeannette Bohg and Ken Goldberg and Abhinav Gupta and Abhishek Gupta and Dinesh Jayaraman and Joseph J Lim and Jitendra Malik and Roberto Martín-Martín and Subramanian Ramamoorthy and Dorsa Sadigh and Shuran Song and Jiajun Wu and Michael C. Yip and Yuke Zhu and Thomas Kollar and Sergey Levine and Chelsea Finn},
    year    = {2024},
}

@article{luo2024fmb,
  title={FMB: a Functional Manipulation Benchmark for Generalizable Robotic Learning},
  author={Luo, Jianlan and Xu, Charles and Liu, Fangchen and Tan, Liam and Lin, Zipeng and Wu, Jeffrey and Abbeel, Pieter and Levine, Sergey},
  journal={arXiv preprint arXiv:2401.08553},
  year={2024}
}

@misc{shafiullah2023dobbe,
    title={On Bringing Robots Home}, 
    author={Nur Muhammad Mahi Shafiullah and Anant Rai and Haritheja Etukuru and Yiqian Liu and Ishan Misra and Soumith Chintala and Lerrel Pinto},
    year={2023},
    eprint={2311.16098},
    archivePrefix={arXiv},
    primaryClass={cs.RO}
}

@inproceedings{jang2022bc,
  title={Bc-z: Zero-shot task generalization with robotic imitation learning},
  author={Jang, Eric and Irpan, Alex and Khansari, Mohi and Kappler, Daniel and Ebert, Frederik and Lynch, Corey and Levine, Sergey and Finn, Chelsea},
  booktitle={Conference on Robot Learning},
  pages={991--1002},
  year={2022},
  organization={PMLR}
}

@article{mendonca2023structured,
  title={Structured World Models from Human Videos},
  author={Mendonca, Russell and Bahl, Shikhar and Pathak, Deepak},
  journal={CoRL},
  year={2023}
}

@inproceedings{quere_shared_2020,
        address = {Paris, France},
        title = {Shared {Control} {Templates} for {Assistive} {Robotics}},
        language = {en},
        booktitle = {2020 {IEEE} {International} {Conference} on {Robotics} and {Automation} ({ICRA})},
        author = {Quere, Gabriel and Hagengruber, Annette and Iskandar, Maged and Bustamante, Samuel and Leidner, Daniel and Stulp, Freek and Vogel, Joern},
        year = {2020},
        pages = {7},
}

@inproceedings{liu2022robot,
    title = {Robot Learning on the Job: Human-in-the-Loop Autonomy and Learning During Deployment},
    author = {Huihan Liu and Soroush Nasiriany and Lance Zhang and Zhiyao Bao and Yuke Zhu},
    booktitle = {Robotics: Science and Systems (RSS)},
    year = {2023}
}

@inproceedings{nasiriany2022sailor,
      title={Learning and Retrieval from Prior Data for Skill-based Imitation Learning},
      author={Soroush Nasiriany and Tian Gao and Ajay Mandlekar and Yuke Zhu},
      booktitle={Conference on Robot Learning (CoRL)},
      year={2022}
    }

@inproceedings{heo2023furniturebench,
  title={FurnitureBench: Reproducible Real-World Benchmark for Long-Horizon Complex Manipulation},
  author={Minho Heo and Youngwoon Lee and Doohyun Lee and Joseph J. Lim},
  booktitle={Robotics: Science and Systems},
  year={2023}
}

@article{cui2022play,
  title   = {From Play to Policy: Conditional Behavior Generation from Uncurated Robot Data},
  author  = {Cui, Zichen Jeff and Wang, Yibin and Shafiullah, Nur Muhammad Mahi and Pinto, Lerrel},
  journal = {arXiv preprint arXiv:2210.10047},
  year    = {2022}
}

@article{belkhale2023hydra,
 title={HYDRA: Hybrid Robot Actions for Imitation Learning},
 author={Belkhale, Suneel and Cui, Yuchen and Sadigh, Dorsa},
 journal={arxiv},
 year={2023}
}

@article{lynch2023interactive,
  title={Interactive language: Talking to robots in real time},
  author={Lynch, Corey and Wahid, Ayzaan and Tompson, Jonathan and Ding, Tianli and Betker, James and Baruch, Robert and Armstrong, Travis and Florence, Pete},
  journal={IEEE Robotics and Automation Letters},
  year={2023},
  publisher={IEEE}
}

@misc{BerkeleyUR5Website,
  title = {Berkeley {UR5} Demonstration Dataset},
  author = {Lawrence Yunliang Chen and Simeon Adebola and Ken Goldberg},
  howpublished = {\url{https://sites.google.com/view/berkeley-ur5/home}},
}

@misc{dass2023jacoplay,
  author = {Dass, Shivin and Yapeter, Jullian and Zhang, Jesse and Zhang, Jiahui
            and Pertsch, Karl and Nikolaidis, Stefanos and Lim, Joseph J.},
  title = {{CLVR} Jaco Play Dataset},
  url = {https://github.com/clvrai/clvr_jaco_play_dataset},
  version = {1.0.0},
  year = {2023}
}

@inproceedings{rosetebeas2022latent,
author = {Erick Rosete-Beas and Oier Mees and Gabriel Kalweit and Joschka Boedecker and Wolfram Burgard},
title = {Latent Plans for Task Agnostic Offline Reinforcement Learning},
booktitle = {Proceedings of the 6th Conference on Robot Learning (CoRL)},
year = {2022}
}

@inproceedings{mees2023grounding,
  title={Grounding  Language  with  Visual  Affordances  over  Unstructured  Data},
  author={Oier Mees and Jessica Borja-Diaz and Wolfram Burgard},
  booktitle = {Proceedings of the IEEE International Conference on Robotics and Automation (ICRA)},
  year={2023},
  address = {London, UK}
}

@article{ebert2021bridge,
  title={Bridge data: Boosting generalization of robotic skills with cross-domain datasets},
  author={Ebert, Frederik and Yang, Yanlai and Schmeckpeper, Karl and Bucher, Bernadette and Georgakis, Georgios and Daniilidis, Kostas and Finn, Chelsea and Levine, Sergey},
  journal={arXiv preprint arXiv:2109.13396},
  year={2021}
}

@inproceedings{kalashnikov2018scalable,
  title={Qt-opt: Scalable deep reinforcement learning for vision-based robotic manipulation},
  author={Kalashnikov, Dmitry and Irpan, Alex and Pastor, Peter and Ibarz, Julian and Herzog, Alexander and Jang, Eric and Quillen, Deirdre and Holly, Ethan and Kalakrishnan, Mrinal and Vanhoucke, Vincent and others},
  booktitle={CoRL},
  pages={651--673},
  year={2018}
}

@article{brohan2022rt1,
  title     = {RT-1: Robotics Transformer for Real-World Control at Scale},
  author    = {Anthony Brohan and Noah Brown and Justice Carbajal and Yevgen Chebotar and Joseph Dabis and Chelsea Finn and K. Gopalakrishnan and Karol Hausman and Alexander Herzog and Jasmine Hsu and Julian Ibarz and Brian Ichter and A. Irpan and Tomas Jackson and Sally Jesmonth and Nikhil J. Joshi and Ryan C. Julian and Dmitry Kalashnikov and Yuheng Kuang and Isabel Leal and Kuang-Huei Lee and S. Levine and Yao Lu and U. Malla and D. Manjunath and Igor Mordatch and Ofir Nachum and Carolina Parada and Jodilyn Peralta and Emily Perez and Karl Pertsch and Jornell Quiambao and Kanishka Rao and M. Ryoo and Grecia Salazar and Pannag R. Sanketi and Kevin Sayed and Jaspiar Singh and S. Sontakke and Austin Stone and Clayton Tan and Huong Tran and Vincent Vanhoucke and Steve Vega and Q. Vuong and F. Xia and Ted Xiao and Peng Xu and Sichun Xu and Tianhe Yu and Brianna Zitkovich},
  journal   = {Robotics: Science and Systems},
  year      = {2022},
  doi       = {10.48550/arXiv.2212.06817},
  bibSource = {Semantic Scholar https://www.semanticscholar.org/paper/fd1cf28a2b8caf2fe29af5e7fa9191cecfedf84d}
}

@InProceedings{goyal2017something,
author = {Goyal, Raghav and Ebrahimi Kahou, Samira and Michalski, Vincent and Materzynska, Joanna and Westphal, Susanne and Kim, Heuna and Haenel, Valentin and Fruend, Ingo and Yianilos, Peter and Mueller-Freitag, Moritz and Hoppe, Florian and Thurau, Christian and Bax, Ingo and Memisevic, Roland},
title = {The "Something Something" Video Database for Learning and Evaluating Visual Common Sense},
booktitle = {Proceedings of the IEEE International Conference on Computer Vision (ICCV)},
month = {Oct},
year = {2017}
}

@inproceedings{octo_2023,
    title={Octo: An Open-Source Generalist Robot Policy},
    author = {{Octo Model Team} and Dibya Ghosh and Homer Walke and Karl Pertsch and Kevin Black and Oier Mees and Sudeep Dasari and Joey Hejna and Charles Xu and Jianlan Luo and Tobias Kreiman and {You Liang} Tan and Lawrence Yunliang Chen and Pannag Sanketi and Quan Vuong and Ted Xiao and Dorsa Sadigh and Chelsea Finn and Sergey Levine},
    booktitle = {Proceedings of Robotics: Science and Systems},
    address  = {Delft, Netherlands},
    year = {2024},
}

@InProceedings{Li_2022_CVPR,
      title = {Egocentric Prediction of Action Target in 3D},
      author = {Li, Yiming and Cao, Ziang and Liang, Andrew and Liang, Benjamin and Chen, Luoyao and Zhao, Hang and Feng, Chen},
      booktitle = {Proceedings of the IEEE/CVF Conference on Computer Vision and Pattern Recognition (CVPR)},
      month = {June},
      year = {2022}
}

@misc{wang2024hocapcapturedataset3d,
      title={HO-Cap: A Capture System and Dataset for 3D Reconstruction and Pose Tracking of Hand-Object Interaction},
      author={Jikai Wang and Qifan Zhang and Yu-Wei Chao and Bowen Wen and Xiaohu Guo and Yu Xiang},
      year={2024},
      eprint={2406.06843},
      archivePrefix={arXiv},
      primaryClass={cs.CV},
      url={https://arxiv.org/abs/2406.06843},
}

@InProceedings{Liu_2022_CVPR,
    author    = {Liu, Yunze and Liu, Yun and Jiang, Che and Lyu, Kangbo and Wan, Weikang and Shen, Hao and Liang, Boqiang and Fu, Zhoujie and Wang, He and Yi, Li},
    title     = {HOI4D: A 4D Egocentric Dataset for Category-Level Human-Object Interaction},
    booktitle = {Proceedings of the IEEE/CVF Conference on Computer Vision and Pattern Recognition (CVPR)},
    month     = {June},
    year      = {2022},
    pages     = {21013-21022}
}

@InProceedings{HoloAssist2023,
    author    = {Wang, Xin and Kwon, Taein and Rad, Mahdi and Pan, Bowen and Chakraborty, Ishani and Andrist, Sean and Bohus, Dan and Feniello, Ashley and Tekin, Bugra and Frujeri, Felipe Vieira and Joshi, Neel and Pollefeys, Marc},
    title     = {HoloAssist: an Egocentric Human Interaction Dataset for Interactive AI Assistants in the Real World},
    booktitle = {Proceedings of the IEEE/CVF International Conference on Computer Vision (ICCV)},
    month     = {October},
    year      = {2023},
    pages     = {20270-20281}
}

@inproceedings{
    fang2023rh20t,
    title = {RH20T: A Robotic Dataset for Learning Diverse Skills in One-Shot},
    author = {Fang, Hao-Shu and Fang, Hongjie and Tang, Zhenyu and Liu, Jirong and Wang, Junbo and Zhu, Haoyi and Lu, Cewu},
    booktitle = {RSS 2023 Workshop on Learning for Task and Motion Planning},
    year = {2023}
}

@misc{pei2025modeling,
      title={Modeling Fine-Grained Hand-Object Dynamics for Egocentric Video Representation Learning}, 
      author={Baoqi Pei and Yifei Huang and Jilan Xu and Guo Chen and Yuping He and Lijin Yang and Yali Wang and Weidi Xie and Yu Qiao and Fei Wu and Limin Wang},
      year={2025},
      eprint={2503.00986},
      archivePrefix={arXiv},
      primaryClass={cs.CV},
      url={https://arxiv.org/abs/2503.00986}, 
}

@misc{mikolov2013distributed,
      title={Distributed Representations of Words and Phrases and their Compositionality}, 
      author={Tomas Mikolov and Ilya Sutskever and Kai Chen and Greg Corrado and Jeffrey Dean},
      year={2013},
      eprint={1310.4546},
      archivePrefix={arXiv},
      primaryClass={cs.CL},
      url={https://arxiv.org/abs/1310.4546}, 
}

@misc{mikolov2013efficient,
      title={Efficient Estimation of Word Representations in Vector Space}, 
      author={Tomas Mikolov and Kai Chen and Greg Corrado and Jeffrey Dean},
      year={2013},
      eprint={1301.3781},
      archivePrefix={arXiv},
      primaryClass={cs.CL},
      url={https://arxiv.org/abs/1301.3781}, 
}

@misc{wieting2015paraphrase,
      title={From Paraphrase Database to Compositional Paraphrase Model and Back}, 
      author={John Wieting and Mohit Bansal and Kevin Gimpel and Karen Livescu and Dan Roth},
      year={2015},
      eprint={1506.03487},
      archivePrefix={arXiv},
      primaryClass={cs.CL},
      url={https://arxiv.org/abs/1506.03487}, 
}

@inproceedings{AroraLM17,
  author       = {Sanjeev Arora and
                  Yingyu Liang and
                  Tengyu Ma},
  title        = {A Simple but Tough-to-Beat Baseline for Sentence Embeddings},
  booktitle    = {5th International Conference on Learning Representations, {ICLR} 2017,
                  Toulon, France, April 24-26, 2017, Conference Track Proceedings},
  publisher    = {OpenReview.net},
  year         = {2017},
  url          = {https://openreview.net/forum?id=SyK00v5xx},
  bibsource    = {dblp computer science bibliography, https://dblp.org}
}

@misc{trager2024linear,
      title={Linear Spaces of Meanings: Compositional Structures in Vision-Language Models}, 
      author={Matthew Trager and Pramuditha Perera and Luca Zancato and Alessandro Achille and Parminder Bhatia and Stefano Soatto},
      year={2024},
      eprint={2302.14383},
      archivePrefix={arXiv},
      primaryClass={cs.LG},
      url={https://arxiv.org/abs/2302.14383}, 
}

@misc{berasi2025text,
      title={Not Only Text: Exploring Compositionality of Visual Representations in Vision-Language Models}, 
      author={Davide Berasi and Matteo Farina and Massimiliano Mancini and Elisa Ricci and Nicola Strisciuglio},
      year={2025},
      eprint={2503.17142},
      archivePrefix={arXiv},
      primaryClass={cs.CV},
      url={https://arxiv.org/abs/2503.17142}, 
}

@misc{gao2025adaworld,
      title={AdaWorld: Learning Adaptable World Models with Latent Actions}, 
      author={Shenyuan Gao and Siyuan Zhou and Yilun Du and Jun Zhang and Chuang Gan},
      year={2025},
      eprint={2503.18938},
      archivePrefix={arXiv},
      primaryClass={cs.AI},
      url={https://arxiv.org/abs/2503.18938}, 
}

@misc{nikulin2025latent,
      title={Latent Action Learning Requires Supervision in the Presence of Distractors}, 
      author={Alexander Nikulin and Ilya Zisman and Denis Tarasov and Nikita Lyubaykin and Andrei Polubarov and Igor Kiselev and Vladislav Kurenkov},
      year={2025},
      eprint={2502.00379},
      archivePrefix={arXiv},
      primaryClass={cs.CV},
      url={https://arxiv.org/abs/2502.00379}, 
}

@inproceedings{zhang2025what,
author = {Zhang, Chuheng and Pearce, Tim and Zhang, Pushi and Wang, Kaixin and Chen, Xiaoyu and Shen, Wei and Zhao, Li and Bian, Jiang},
title = {What Do Latent Action Models Actually Learn?},
booktitle = {NeurIPS 2025},
year = {2025},
month = {May},
url = {https://www.microsoft.com/en-us/research/publication/what-do-latent-action-models-actually-learn/},
}

@misc{chen2025villax,
      title={villa-X: Enhancing Latent Action Modeling in Vision-Language-Action Models}, 
      author={Xiaoyu Chen and Hangxing Wei and Pushi Zhang and Chuheng Zhang and Kaixin Wang and Yanjiang Guo and Rushuai Yang and Yucen Wang and Xinquan Xiao and Li Zhao and Jianyu Chen and Jiang Bian},
      year={2025},
      eprint={2507.23682},
      archivePrefix={arXiv},
      primaryClass={cs.RO},
      url={https://arxiv.org/abs/2507.23682}, 
}

@article{liang2025clam,
  title={Clam: Continuous latent action models for robot learning from unlabeled demonstrations},
  author={Liang, Anthony and Czempin, Pavel and Hong, Matthew and Zhou, Yutai and Biyik, Erdem and Tu, Stephen},
  journal={arXiv preprint arXiv:2505.04999},
  year={2025}
}

@misc{yang2025como,
      title={CoMo: Learning Continuous Latent Motion from Internet Videos for Scalable Robot Learning}, 
      author={Jiange Yang and Yansong Shi and Haoyi Zhu and Mingyu Liu and Kaijing Ma and Yating Wang and Gangshan Wu and Tong He and Limin Wang},
      year={2025},
      eprint={2505.17006},
      archivePrefix={arXiv},
      primaryClass={cs.CV},
      url={https://arxiv.org/abs/2505.17006}, 
}

@misc{garrido2026,
      title={Learning Latent Action World Models In The Wild}, 
      author={Quentin Garrido and Tushar Nagarajan and Basile Terver and Nicolas Ballas and Yann LeCun and Michael Rabbat},
      year={2026},
      eprint={2601.05230},
      archivePrefix={arXiv},
      primaryClass={cs.AI},
      url={https://arxiv.org/abs/2601.05230}, 
}

@misc{ren2025videoworld,
      title={VideoWorld: Exploring Knowledge Learning from Unlabeled Videos}, 
      author={Zhongwei Ren and Yunchao Wei and Xun Guo and Yao Zhao and Bingyi Kang and Jiashi Feng and Xiaojie Jin},
      year={2025},
      eprint={2501.09781},
      archivePrefix={arXiv},
      primaryClass={cs.CV},
      url={https://arxiv.org/abs/2501.09781}, 
}

@misc{cai2025SSMVLA,
      title={Seeing Space and Motion: Enhancing Latent Actions with Spatial and Dynamic Awareness for VLA}, 
      author={Zhejia Cai and Yandan Yang and Xinyuan Chang and Shiyi Liang and Ronghan Chen and Feng Xiong and Mu Xu and Ruqi Huang},
      year={2025},
      eprint={2509.26251},
      archivePrefix={arXiv},
      primaryClass={cs.CV},
      url={https://arxiv.org/abs/2509.26251}, 
}

@misc{bu2025laof,
      title={LAOF: Robust Latent Action Learning with Optical Flow Constraints}, 
      author={Xizhou Bu and Jiexi Lyu and Fulei Sun and Ruichen Yang and Zhiqiang Ma and Wei Li},
      year={2025},
      eprint={2511.16407},
      archivePrefix={arXiv},
      primaryClass={cs.RO},
      url={https://arxiv.org/abs/2511.16407}, 
}

@misc{zhang2026clap,
      title={CLAP: Contrastive Latent Action Pretraining for Learning Vision-Language-Action Models from Human Videos}, 
      author={Chubin Zhang and Jianan Wang and Zifeng Gao and Yue Su and Tianru Dai and Cai Zhou and Jiwen Lu and Yansong Tang},
      year={2026},
      eprint={2601.04061},
      archivePrefix={arXiv},
      primaryClass={cs.RO},
      url={https://arxiv.org/abs/2601.04061}, 
}

@misc{bi2025motus,
      title={Motus: A Unified Latent Action World Model}, 
      author={Hongzhe Bi and Hengkai Tan and Shenghao Xie and Zeyuan Wang and Shuhe Huang and Haitian Liu and Ruowen Zhao and Yao Feng and Chendong Xiang and Yinze Rong and Hongyan Zhao and Hanyu Liu and Zhizhong Su and Lei Ma and Hang Su and Jun Zhu},
      year={2025},
      eprint={2512.13030},
      archivePrefix={arXiv},
      primaryClass={cs.CV},
      url={https://arxiv.org/abs/2512.13030}, 
}

@misc{routray2025vipra,
      title={ViPRA: Video Prediction for Robot Actions}, 
      author={Sandeep Routray and Hengkai Pan and Unnat Jain and Shikhar Bahl and Deepak Pathak},
      year={2025},
      eprint={2511.07732},
      archivePrefix={arXiv},
      primaryClass={cs.RO},
      url={https://arxiv.org/abs/2511.07732}, 
}

@book{barfoot2024state,
  title={State estimation for robotics},
  author={Barfoot, Timothy D},
  year={2024},
  publisher={Cambridge University Press}
}

@article{van2017neural,
  title={Neural discrete representation learning},
  author={Van Den Oord, Aaron and Vinyals, Oriol and others},
  journal={Advances in neural information processing systems},
  volume={30},
  year={2017}
}

@inproceedings{grill2020byol,
 author = {Grill, Jean-Bastien and Strub, Florian and Altch\'{e}, Florent and Tallec, Corentin and Richemond, Pierre and Buchatskaya, Elena and Doersch, Carl and Avila Pires, Bernardo and Guo, Zhaohan and Gheshlaghi Azar, Mohammad and Piot, Bilal and kavukcuoglu, koray and Munos, Remi and Valko, Michal},
 booktitle = {Advances in Neural Information Processing Systems},
 editor = {H. Larochelle and M. Ranzato and R. Hadsell and M.F. Balcan and H. Lin},
 pages = {21271--21284},
 publisher = {Curran Associates, Inc.},
 title = {Bootstrap Your Own Latent - A New Approach to Self-Supervised Learning},
 url = {https://proceedings.neurips.cc/paper_files/paper/2020/file/f3ada80d5c4ee70142b17b8192b2958e-Paper.pdf},
 volume = {33},
 year = {2020}
}

@misc{he2019momentum,
      title={Momentum Contrast for Unsupervised Visual Representation Learning}, 
      author={Kaiming He and Haoqi Fan and Yuxin Wu and Saining Xie and Ross Girshick},
      year={2020},
      eprint={1911.05722},
      archivePrefix={arXiv},
      primaryClass={cs.CV},
      url={https://arxiv.org/abs/1911.05722}, 
}

@article{mnih2015humanlevel,
  author = {Mnih, Volodymyr and Kavukcuoglu, Koray and Silver, David and Rusu, Andrei A. and Veness, Joel and Bellemare, Marc G. and Graves, Alex and Riedmiller, Martin and Fidjeland, Andreas K. and Ostrovski, Georg and Petersen, Stig and Beattie, Charles and Sadik, Amir and Antonoglou, Ioannis and King, Helen and Kumaran, Dharshan and Wierstra, Daan and Legg, Shane and Hassabis, Demis},
  description = {Human-level control through deep reinforcement learning - nature14236.pdf},
  issn = {00280836},
  journal = {Nature},
  month = feb,
  number = 7540,
  pages = {529--533},
  publisher = {Nature Publishing Group, a division of Macmillan Publishers Limited. All Rights Reserved.},
  title = {Human-level control through deep reinforcement learning},
  url = {http://dx.doi.org/10.1038/nature14236},
  volume = 518,
  year = 2015
}

@misc{zhang2025vlaseffectivelyinheritvlms,
      title={How Do VLAs Effectively Inherit from VLMs?}, 
      author={Chuheng Zhang and Rushuai Yang and Xiaoyu Chen and Kaixin Wang and Li Zhao and Yi Chen and Jiang Bian},
      year={2025},
      eprint={2511.06619},
      archivePrefix={arXiv},
      primaryClass={cs.RO},
      url={https://arxiv.org/abs/2511.06619}, 
}

@misc{wang2025cola_world,
  title        = "Co-Evolving Latent Action World Models",
  author       = "Anonymous Authors",
  howpublished = "Concurrent Submission to ICML",
  year         = 2026,
  note         = "Filename: colaworld.pdf"
}

@misc{MVP-LAM,
  title        = "MVP-LAM: Learning Action-Centric Latent Action via Cross-Viewpoint Reconstruction",
  author       = "Anonymous Authors",
  howpublished = "Concurrent Submission to ICML",
  year         = 2026,
  note         = "Filename: mvplam.pdf"
}

@article{li2025latbot,
  title={LatBot: Distilling Universal Latent Actions for Vision-Language-Action Models},
  author={Li, Zuolei and Gao, Xingyu and Wang, Xiaofan and Fu, Jianlong},
  journal={arXiv preprint arXiv:2511.23034},
  year={2025}
}

@article{chen2016xgboost,
  title={XGBoost: A Scalable Tree Boosting System},
  author={Chen, Tianqi},
  journal={Cornell University},
  year={2016}
}
\bibliographystyle{icml2026}

\newpage
\appendix
\onecolumn




\section{Training Details}

\subsection{LAM Training Details}
\label{appendix:LAM training}



Our LAM design largely follows villa-X LAM, augmented with an additive-composition (AC) loss and dynamic temporal intervals to support AC constraints. The IDM uses 12 Transformer encoder layers. Given an image pair $(o_i, o_j)$ (default $2\times3\times224\times224$), we apply a patch embedding with patch size 14, concatenate image tokens and stack 12 self-attention blocks (hidden dimension 768, 32 attention heads). The FDM is a 12-layer Vision Transformer (ViT-Base) that predicts $o_j$ from $(o_i, z_{ij})$. Following villa-X, we also employ a proprioceptive FDM: a 2-layer MLP with dual output heads that predict future robot states $q_j$, conditioned on $(q_i, z_{ij})$.


AC-LAM is trained on a mixture of human egocentric videos (e.g., Ego4D~[21]) and robot trajectories (e.g., OpenX~[12]). For scene-wise AC sampling, we draw triples $(i,j,k)$ from the same trajectory: robot temporal offsets are sampled uniformly from $[0.1, 3]\,\mathrm{s}$ and human offsets from $[0.1, 2]\,\mathrm{s}$. We further filter robot triples exhibiting large rotations so that additive composition remains a reasonable approximation. 
Given the inherent temporal smoothness of robot motion, this approximation effectively captures dynamics within the proposed time range.
We use a batch size of 512 and a learning rate of $1.5\times10^{-4}$ with a 2000-step linear warmup. Training lasts approximately 10 days on 32 NVIDIA A100 GPUs.

\subsection{Policy Training Details}
\label{appendix:policy training}


We select villa-x as the policy architecture for downstream policy learning, to assess latent action model's ability to provide high-quality supervision signals. 
The policy model in villa-x comprises three components. First, the vision–language encoder is based on PaliGemma[3], a 3B-parameter VLM pretrained with 224 × 224 images and 128-token text inputs. Second and third, the latent-action expert and the robot-action expert are each implemented as 18-layer Transformer networks, mirroring PaliGemma’s design, with a hidden dimension of 1,024 and 8 attention heads. For the latent action sequence, we select a sequence length of N = 6, and for the robot actions, we select a sequence length of M = 4. We apply the same random attention mask and random attention dropout techniques as in villa-x. We train all components jointly using a learning rate of 5e-5 with a 200-step linear warmup. We clip gradients to a maximum norm of 1.0 to ensure stable optimization. 

We did not pretrain the policy model on large-scale dataset. The goal here is use the model as a convenient policy learning method that can take both latent actions and robot actions as supervision signals. The policy learning follows the training data setup as mentioned in the experiment part. Each policy training with different LAMs takes 15K gradient steps, with a batch size of 512. To assess generalization under cross‑dataset transfer, we randomly form a 50\%/50\% mixture of the in-distribution dataset and Bridge V2 and use this combined corpus for training.

\section{Datasets for Latent Action Learning}

\label{appendix:LAM datasets}

\subsection{Data Mixture}

We follow the data mixture in villa-x, which combines both robot data and action-free human videos for our LAM pretraining phase. For robot data, we draw primarily from OpenX~\citep{open_x_embodiment_rt_x_2023} mixture  and AgiBot~\citep{agibot-world-contributors2025agibot}. For OpenX dataset, our base data mixture is created primarily based on~\citep{kim2024openvla, octo_2023}.
In total, we use 1.6M trajectories with 223.5M frames of robot data.
For human videos, we use a mixture of Ego4D \citep{grauman2022ego4d}, EgoPAT3D~\citep{Li_2022_CVPR}, EGTEA Gaze+~\citep{li2018eye}, EPIC-KITCHENS~\citep{damen2020epic}, HO-Cap~\citep{wang2024hocapcapturedataset3d}, HOI4D~\citep{Liu_2022_CVPR}, HoloAssist~\citep{HoloAssist2023}, RH20T~\citep{fang2023rh20t}, Something Something V2~\citep{goyal2017something}. 
Altogether, this yields 3.6M clips of human videos. 
During LAM pretraining, we exclusively utilize the primary third-person camera view. 
A full breakdown of our data mixture is listed in Table \ref{table:dataset}.

\subsection{Data Preprocessing}

For data cleaning, we adopt EgoHOD \citep{pei2025modeling}, a curated subset of Ego4D \citep{grauman2022ego4d}, and further filter the videos based on visual quality to ensure high-quality inputs for training. For both robot data and human videos, we apply random adjustments to brightness, contrast, saturation, and hue as data augmentation. In the case of robot data, we represent both proprioceptive states and actions using euler angles.

\begin{table}[htb!]
\centering
\begin{tabular}{cc}\toprule  
Dataset                                                         & Mix Ratio (\%) \\ \midrule
RT-1 Robot Action \citep{brohan2022rt1}                         & 9.70           \\
AgiBot World Beta \citep{agibot-world-contributors2025agibot}   & 20.0           \\
Kuka   \citep{kalashnikov2018scalable}                          & 1.97           \\
Bridge \citep{walke2023bridgedata,ebert2021bridge}              & 5.47           \\
Taco Play \citep{rosetebeas2022latent,mees2023grounding}        & 0.76           \\
Jaco Play  \citep{dass2023jacoplay}                             & 0.12           \\
Berkely Autolab UR5 \citep{BerkeleyUR5Website}                  & 0.31           \\
Language Table  \citep{lynch2023interactive}                  & 0.11           \\
Stanford Hydra Dataset \citep{belkhale2023hydra}                & 1.61           \\
NYU Franka Play Dataset \citep{cui2022play}                     & 0.22           \\
Furniture Bench Dataset  \citep{heo2023furniturebench}          & 0.63           \\
Austin Sailor Dataset \citep{nasiriany2022sailor}               & 0.57           \\
Austin Sirius Dataset  \citep{liu2022robot}                     & 0.45           \\
BC-Z \citep{jang2022bc}                                         & 3.47           \\
DLR EDAN Shared Control \citep{quere_shared_2020}               & 0.01           \\
CMU Stretch \citep{mendonca2023structured}                      & 0.04           \\
FMB Dataset \citep{luo2024fmb}                                  & 0.73           \\
DobbE \citep{shafiullah2023dobbe}                               & 0.37           \\
DROID \citep{khazatsky2024droid}                                & 3.46           \\
Ego4D \citep{grauman2022ego4d,pei2025modeling}                                 & 21.46           \\
EgoPAT3D~\citep{Li_2022_CVPR}                                     & 0.94           \\
EGTEA Gaze+~\citep{li2018eye}                                 & 0.89           \\
EPIC-KITCHENS~\citep{damen2020epic}                            & 6.95           \\
HO-Cap~\citep{wang2024hocapcapturedataset3d}                     & 0.63           \\
HOI4D~\citep{Liu_2022_CVPR}                                      & 1.99           \\
HoloAssist~\citep{HoloAssist2023}                                & 4.77           \\
RH20T~\citep{fang2023rh20t}                                      & 5.56           \\
Something-Something-V2~\citep{goyal2017something}               & 6.82           \\ \bottomrule
\end{tabular}\caption{ Our training data mixture used in LAM training. 
}
\label{table:dataset}
\end{table}

\section{More Details for Experiments on Latent Action Structure}
\label{appendix:latent_action_structure}


\paragraph{Sampling latent actions for Eq. 8 to 10}  
We adopt the same scene-wise $(i,j,k)$ sampling procedure used during training. For each dataset, we draw 16k latent-action instances (pairs or triplets, as required by each metric). Metrics are computed per instance and then averaged to approximate the corresponding expectations.

\paragraph{Alignment between latent-norm and true motion magnitude }



We first compute the motion magnitude as the Euclidean distance between the proprioceptive states $s_i$,$s_j$ to obtain $\|\Delta s_{ij}\|$
Next, we rescale $\|\Delta s_{ij}\|$ per dimension using dataset quantiles: values are normalized with respect to the 1st and 99th percentiles, with clipping below the 1st percentile and above the 99th to reduce the influence of outliers.
To make the Pearson correlation objective more stable and differentiable, we uniformly sampled $(i,j)$ pairs by motion magnitude so that the dataset spans a broad range of $\|\Delta s_{ij}\|$.

\begin{equation}
r\big(\|z\|,\,|\Delta s|\big)
= \frac{\operatorname{cov}\!\big(\|z\|,\,\|\Delta s\|\big)}{\sigma_{\|z\|}\,\sigma_{\|\Delta s\|}}
\end{equation}


\paragraph{Quantifying environment-specific leakage} We assess environment-specific leakage by training a simple latent probe to predict the data source (environment) from latent actions. From the latent probe dataset, we sample 100×32 latent action instances per environment across Fractal, Bridge, Kuka, and DROID. An XGBoost classifier is trained to predict the environment label from these latents, using a random 80/20 train/test split. We report test-set accuracy as the leakage metric, denoted $\mathrm{Acc}_{\mathrm{env}}^{\mathrm{mlp}}$. Higher $\mathrm{Acc}_{\mathrm{env}}^{\mathrm{mlp}}$.
indicates stronger environment-identifying signals present in the latents (i.e., greater leakage), whereas lower $\mathrm{Acc}_{\mathrm{env}}^{\mathrm{mlp}}$ suggests more environment-agnostic representations. All results are reported on the held-out 20\% test split.

\paragraph{More details for ablations on different design choices in AC-LAM}

We consider two variants of applying the stop-gradient mechanism to $\mathcal{L}_{\mathrm{AC\text{-}IDM}}$:
\[
\mathcal{L}_{\mathrm{AC\text{-}IDM}} = \| \operatorname{sg}(f(o_i,o_k)) - (f(o_i,o_j) + f(o_j,o_k)) \|_2^2
\]
and
\[
\mathcal{L}_{\mathrm{AC\text{-}IDM}} = \| f(o_i,o_k) - \operatorname{sg}(f(o_i,o_j) + f(o_j,o_k)) \|_2^2.
\]

For our pre-VQ ablations, the AC loss is applied to the continuous latent vectors immediately following the IDM encoder, prior to the discretization bottleneck of the vector quantizer. This contrasts with our default post-VQ approach, which constrains the quantized codebook embeddings. All ablation models are evaluated using the same sampling method with our main experiments.

\section{More Details for Benchmarks}
\label{appendix:benchmarks}


\begin{figure}[t]
  \centering
  \begin{subfigure}[t]{0.6\linewidth}
    \centering
    \includegraphics[width=\linewidth]{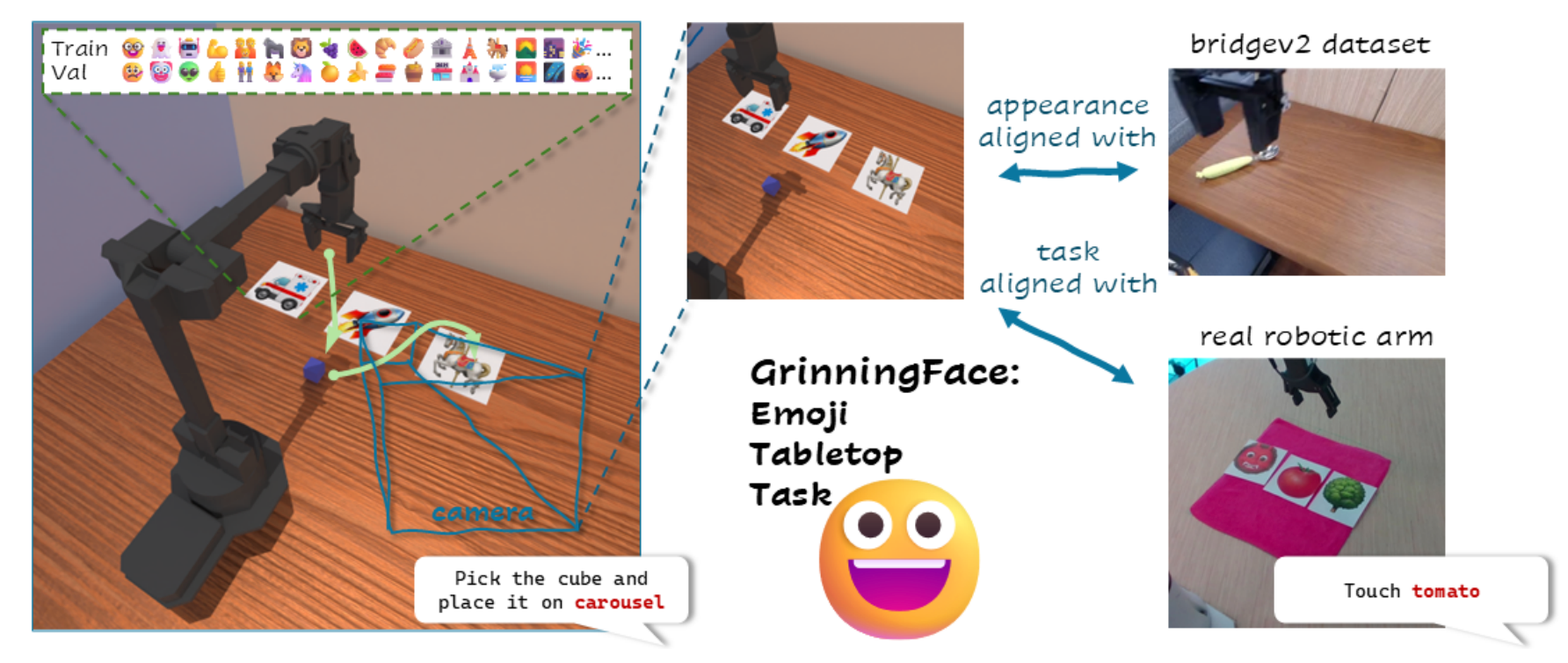}
    \caption{Emoji Table-Top (GrinningFace)}
    \label{fig:left}
  \end{subfigure}\hfill
  \begin{subfigure}[t]{0.35\linewidth}
    \centering
    \includegraphics[width=\linewidth]{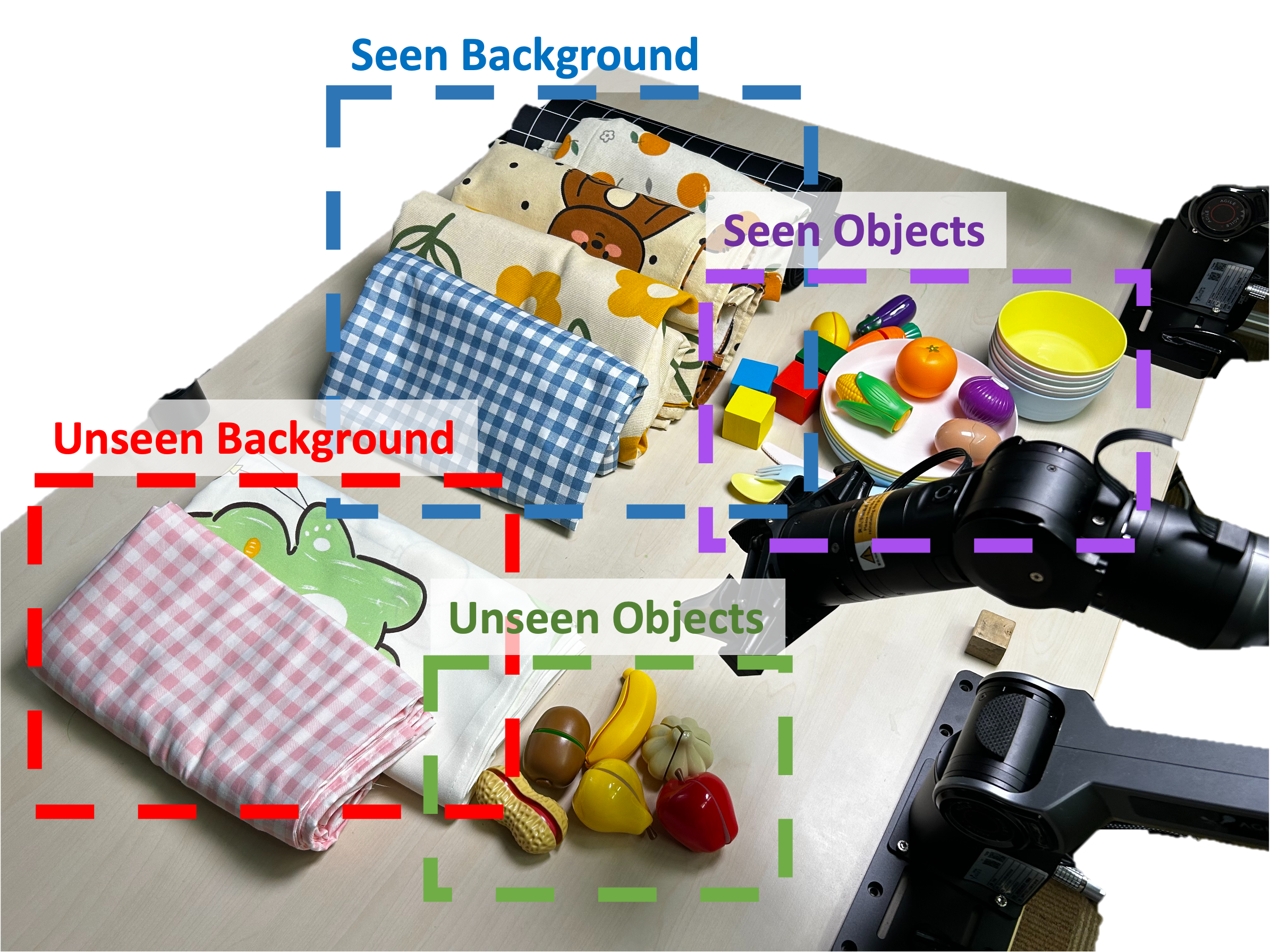}
    \caption{Real-World Tabletop Manipulation}
    \label{fig:right}
  \end{subfigure}
  \caption{Two experimental environments: (a) Emoji Table-Top (GrinningFace) simulation for controlled studies of vision–semantic generalization. A robotic arm picks a cube and places it on the instructed emoji. The viewpoint is aligned with Bridge‑v2 to leverage this large-scale dataset for knowledge transfer. 
  The initial positions of the cube, the emojis, and the robotic arm are randomized to test robustness.   (b) Real‑World Tabletop Manipulation featuring diverse pick tasks across varied objects and backgrounds; evaluations cover in‑distribution scenes, OOD distractors, and OOD backgrounds to assess robustness.}
  \label{fig:two-panel}
\end{figure}

\paragraph{Emoji Table‑Top (GrinningFace)} As shown in Figure~\ref{fig:left}~(we use the figure from \cite{zhang2025vlaseffectivelyinheritvlms}), it is a diagnostic simulation benchmark targets vision–semantic generalization in embodied control, evaluating how vision–language action models inherit priors from vision–language models. Each episode follows the instruction template “Pick the cube and place it on [desc.]”, where [desc.] is the language description of the target emoji. Three emoji cards are placed on the tabletop; success requires grasping the cube and placing it on the correct target card. To leverage Bridge dataset and enable knowledge transfer, the camera viewpoint is aligned with Bridge‑v2~\cite{walke2023bridgedata}.
The initial positions of the cube, the emojis, and the robotic arm are randomized to systematically test robustness. Evaluation follows three protocols defined by the benchmark: ID (in‑distribution combinations and order), Train (novel combinations composed from training‑set emojis), and Val (held‑out validation emojis that are out‑of‑distribution).

\textbf{Real‑World Tabletop Manipulation} A physical setup for evaluating accurate control and robustness under realistic variability. Experiments use an AgileX Robotics Piper arm featuring a 7‑DoF action space and focus on diverse pick tasks across varied objects and backgrounds. The dataset comprises 170 teleoperated trajectories, collected under varied tabletop settings—including different tablecloth textures/colors, object layouts, and object positions—to increase scene diversity and support robustness evaluation. Evaluations cover (i) in‑distribution scenes, (ii) OOD distractors (novel or repositioned non‑target objects), and (iii) OOD backgrounds (changes to tabletop/background appearance), enabling a comprehensive assessment of robustness.
To ensure statistical reliability, we report success rates averaged over 15 rollouts for each setting.

\section{Visualization}

\subsection{Motion transfer demo with summed latents $z_{ij} + z_{jk}$}

\begin{figure*}[htbp]
    \centering
    \includegraphics[width=0.9\linewidth]{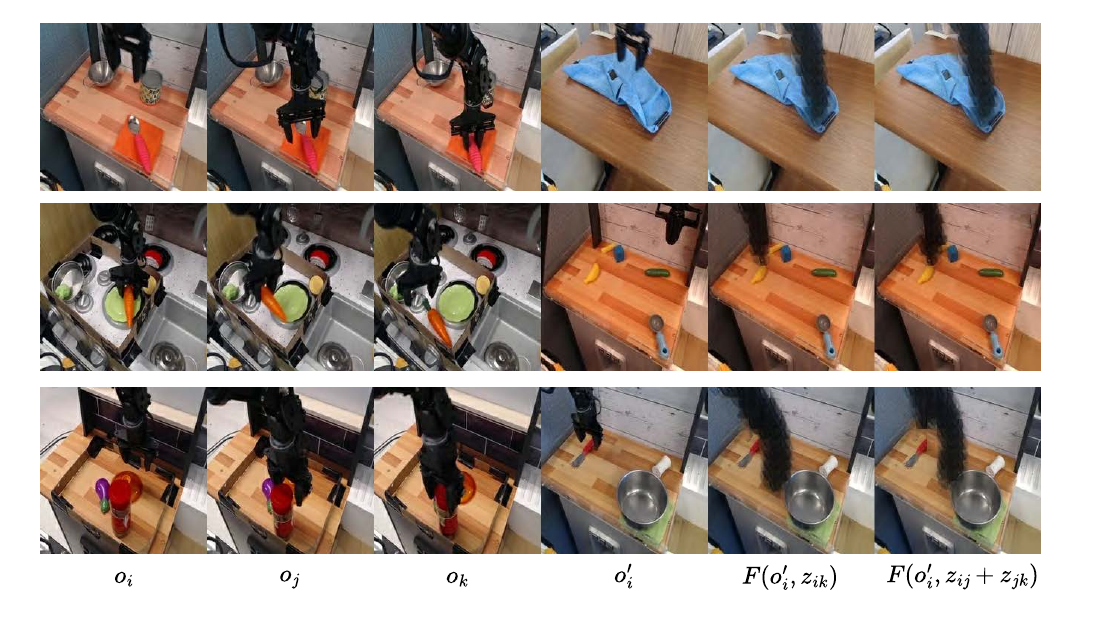}
    \caption{Motion Transfer Demo}
    \label{fig:ac_transfer_demo}
\end{figure*}

Figure \ref{fig:ac_transfer_demo} shows several motion transfer examples. Input frames $o_i$, $o_j$, and $o_k$ are sampled from the same trajectory in the BridgeV2 dataset. Latent actions $z_{ij}$, $z_{jk}$, and $z_{ik}$ are extracted using the IDM and then applied to another sampled frame $o'_i$ via $F(o'_i, z_{ik})$ and $F(o'_i, z_{ij} + z_{jk})$. The results show that the FDM outputs are consistent when using either the direct latent action or the composed latent action, indicating that the semantic meanings of the two paths are well aligned.

\subsection{Additional Latent‑Action Norm Trajectories: Emoji Table‑Top and Real‑World}


Figure~\ref{fig:latent-norm-trajectory-appendix} provides extended visualizations of the latent action norm $\|LAM(o_0,o_t)\|$ evolution in both the Emoji Table-Top (GrinningFace) simulation and real-world tabletop environments, corroborating the analysis in the main text. The trends are consistent: AC-LAM demonstrates the strongest displacement calibration. Villa-X shows a correlation but remains under-calibrated, while LAPA and UniVLA fail to meaningfully track displacement. 
Collectively, these trajectories illustrate that the latent‑action norm generated by AC-LAM provides an interpretable proxy for the amount of motion from the initial observation.

\begin{figure*}[htbp]
    \centering
    \includegraphics[width=0.8\textwidth]{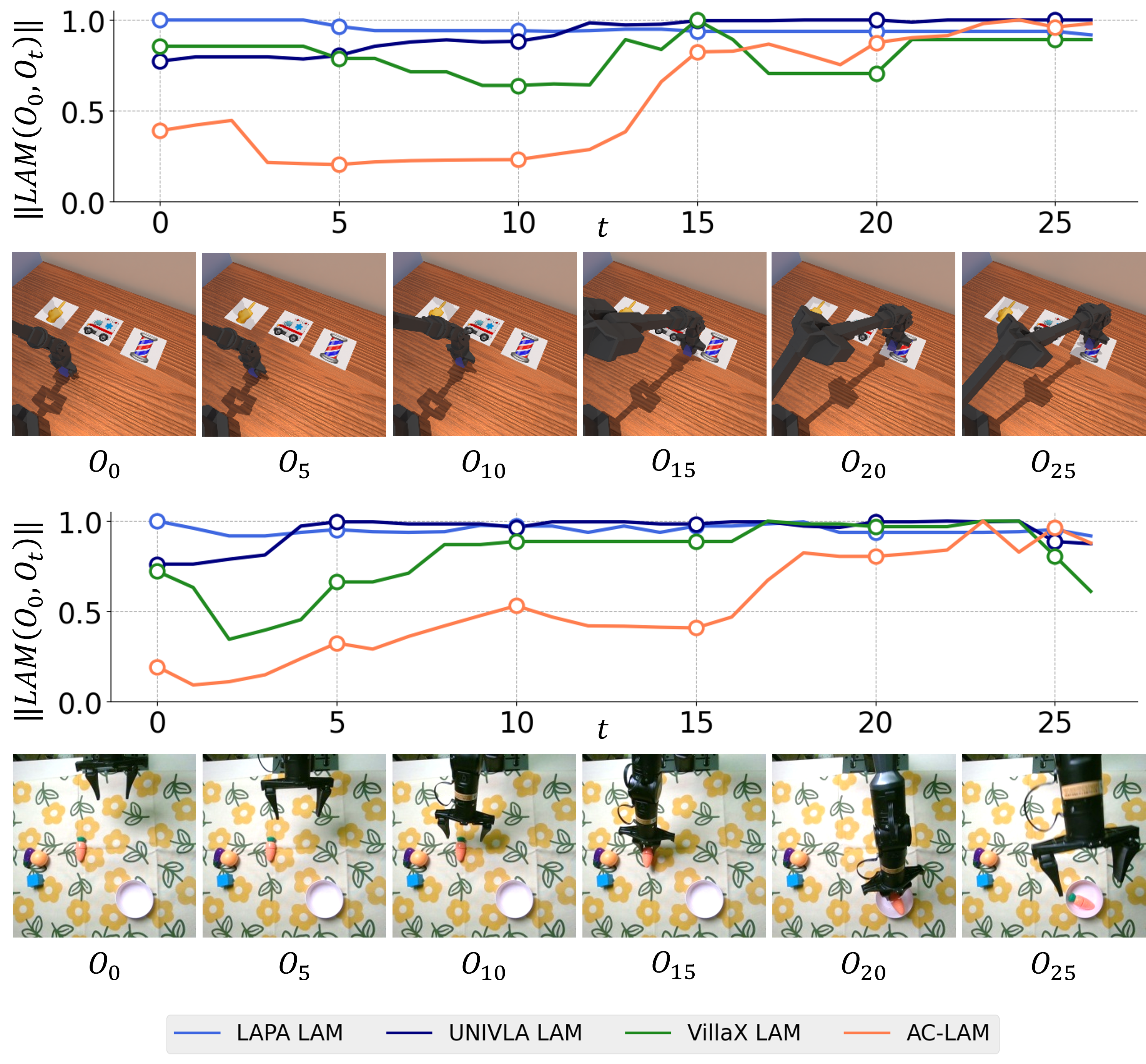}
    \caption{
    More trajectories of the latent action norm $\|LAM(o_0,o_t)\|$ in Emoji Table-Top simulation and real-world tabletop manipulation, with latent actions generated by LAPA LAM, UniVLA LAM, villa-X LAM and AC-LAM. AC‑LAM yields the most displacement‑calibrated latents, aligning with motion magnitude.}
    \label{fig:latent-norm-trajectory-appendix}
\end{figure*}





\end{document}